\documentclass[letterpaper, 10 pt, journal, twoside]{IEEEtran}
\author{Paper ID: }

\date{December 2019}

\usepackage[utf8]{inputenc}
\usepackage{amsmath}
\usepackage[Symbol]{upgreek}
\usepackage{graphics}
\usepackage[tight]{subfigure}
\usepackage{hyperref}
\usepackage[]{footmisc}
\usepackage{comment}
\usepackage{todonotes}
\graphicspath{{./figs/}}
\usepackage{amssymb}
\usepackage{amsfonts}
 \providecommand{\norm}[1]{\lVert#1\rVert}
\usepackage{cite}
\usepackage[ruled,vlined,linesnumbered]{algorithm2e}

\let\proof\relax
\let\endproof\relax
\usepackage{amsthm}
\usepackage{authblk}

\newcommand{\rebut}[1]{\textcolor{black}{{}#1}}
\newcommand{\rebutf}[1]{\textcolor{black}{{}#1}}

\DeclareMathOperator*{\argmax}{arg\,max}

\title{
CoCo Games: Graphical Game-Theoretic Swarm Control for Communication-Aware Coverage
}
\author{Malintha~Fernando, Ransalu~Senanayake, Martin~Swany
\thanks{Malintha Fernando and Martin Swany are with the Luddy School of Informatics, Computing, and Engineering  at Indiana University, Bloomngton, IN, 47401, USA. Ransalu Senanayake is with Stanford University, CA, 94305, USA. E-mail:{\tt\small \{ccfernan, swany\} @iu.edu, ransalu@stanford.edu}.}
\thanks{Code implementations and the video demonstrations for this paper can be found at: \url{https://malintha.github.io/coco}.}
\thanks{Digital Object Identifier (DOI): 10.1109/LRA.2022.3160968}
}

\theoremstyle{definition}

\newtheorem{definition}{Definition}

\newtheorem{theorem}{Theorem}
\newtheorem{remark}{Remark}
\begin{document}

\maketitle
\begin{abstract}
We propose a novel framework for real-time communication-aware coverage control in networked robot swarms.
Our framework unifies the robot dynamics with network-level message-routing to reach consensus on swarm formations in the presence of communication uncertainties by leveraging local information.
\rebutf{Specifically, we formulate the communication-aware coverage as a cooperative graphical game, and use variational inference to reach mixed strategy Nash equilibria of the stage games.} 
We experimentally validate the proposed approach in a mobile ad-hoc wireless network scenario using teams of aerial vehicles and terrestrial user equipment (UE) operating over a large geographic region of interest.
We show that our approach can provide wireless coverage to stationary and mobile UEs under realistic network conditions.
Find the video demonstrations at \url{https://youtu.be/kQJbc_s4ZLI}.
\end{abstract}
\section{Introduction}

\IEEEPARstart{M}{ulti-robot} systems have been gaining significant attention in interdisciplinary research areas such as wireless networks and environmental monitoring thanks to the recent advancements in robotics and telecommunication sectors \cite{mozaffari2019tutorial}.
Specifically, mobile \textit{ad-hoc} wireless networks are emerging as a disruptive technology to accommodate on-demand coverage and capacity enhancement for networked robot systems \cite{sharma2016uav, skorobogatov2020multiple}.
However, deploying robots in such applications requires overcoming numerous challenges; maintaining the connectivity, maximizing the network coverage, and performing real-time motion planning with limited global information, to name a few.

We propose a novel game-theoretical approach to maximize the coverage over a geographical region of interest (ROI) under practical network constraints in real-time.
\rebutf{Specifically, we formulate the communication-aware coverage as a \textit{general-sum graphical game} where the robots aim to maximize two auxiliary objectives: 1) coverage, and 2) connectivity, by leveraging local information.}
The general-sum games permit each robot's payoff to arbitrarily relate, thus making it an ideal framework for coordinating swarms with interconnected objectives like ours.
\rebutf{Additionally, the graphical game-theoretic foundation underpins using the robots' local information to reach swarm consensus.}
Compared to more conventional stochastic games framework which require one's payoff to depend on the joint action profile of \textit{all} the others, this greatly simplifies the game structure \cite{kearns2001graphical}.
\begin{figure}
\centering
        \includegraphics[trim={1.5cm 2cm 2.5cm 4.9cm}, clip, width=0.41\textwidth]{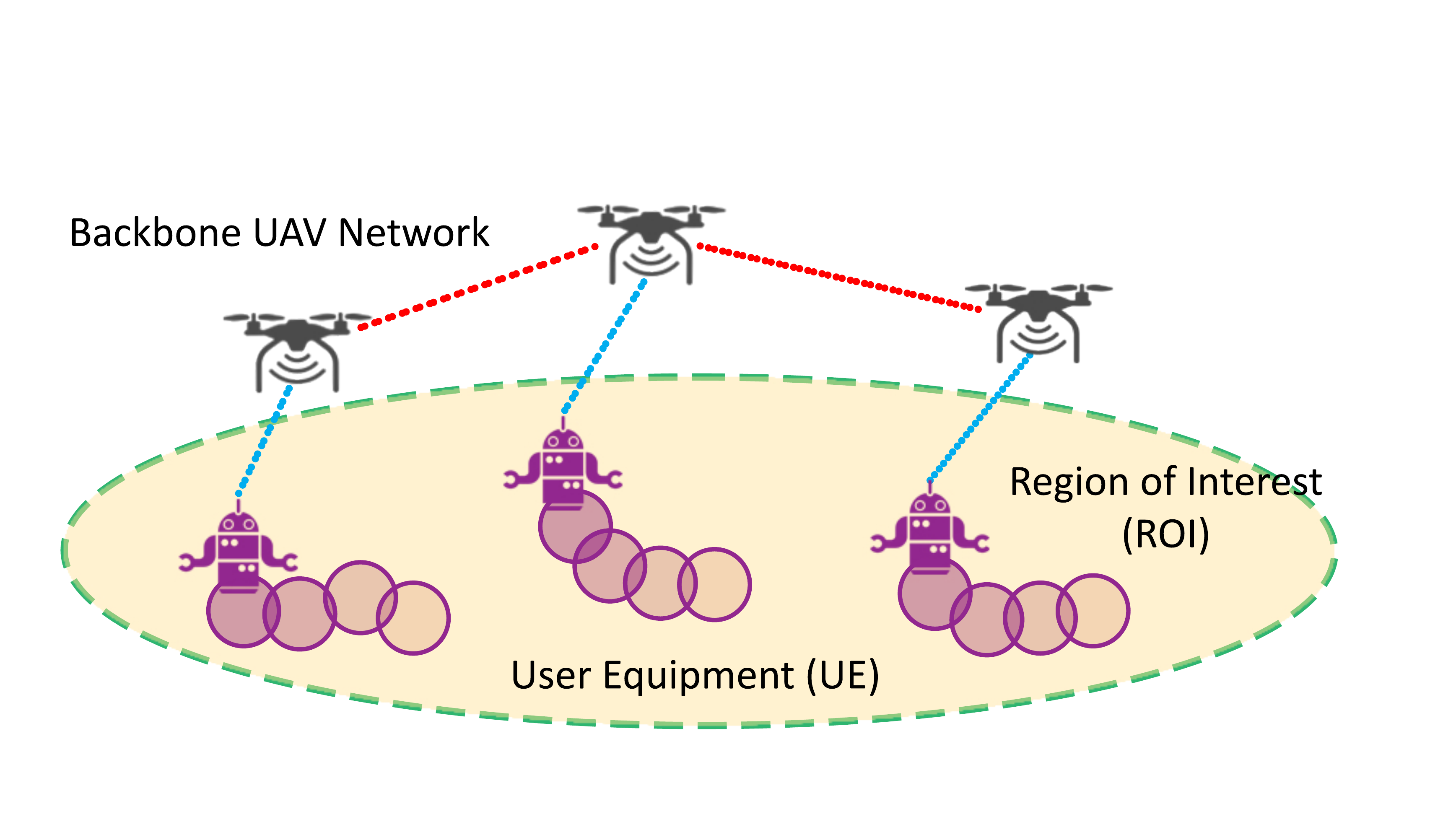}
        \caption{A swarm of 3 UAV robots providing wireless network coverage to a team of 3 travelling UEs. The ROI, UAV-UAV and UAV-UE communication linkes are represented by the ellipsoid, red and blue dashed lines, respectively. 
        CoCo games can coordinate the UAV swarm in real time to maximize the coverage for the ROI independent of the UE movements.}
        \label{fig:cover}
         \vspace{-10pt}
\end{figure}

\rebut {To account for often changing communication topology caused by network uncertainties, we routinely update the game structure with network-level information from message-routing tables.
In contrast to disk-based coverage schemes, this renders our approach highly robust to the volatility of the communication topology caused by the robots' movements and signal attenuation.
While many swarm control methods require aggregating the global information \cite{zhu2019distributed, csenbacslar2019robust}, the local neighborhood property of graphical games allow us to selectively integrate partial observations into decision-making, reducing the communication overhead.
Thus, we believe graphical games allow designing more effective control paradigms for large-scale robot swarms, where the global state aggregation is intractable due to communication limitations.}

By using variational inference (VI), we substantiate the interplay between the game and an adjoined Markov random field (MRF), whose posterior distribution resembles the solution of the former.
We experimentally validate our approach in a mobile wireless network scenario with an Unmanned Aerial Vehicle (UAV) team to provide coverage for a set of User Equipment (UE) over an ROI (Fig. \ref{fig:cover}).

The main contributions of our work are, 1) Formulation of the communication-aware coverage as a graphical game, 2) Theoretical guarantees for the stage-game's equilibrium by leveraging variational inference, 3) A scalable algorithm to reach the equilibrium consensus, and 4) Experimental results for the proposed approach under realistic network conditions.

\section{Related Works}
\subsection{Multi-robot Coverage}
The coverage problem typically involves deploying a set of mobile nodes over a field to maximize some objective: wireless signal coverage -- in case of wireless routers or, information gain -- in case of sensors \cite{li2005distributed, schwager2009decentralized}.
Myriad literature discusses coordinating wireless nodes over the spectrum of communication and control methods ranging from disk-based channel models with centralized control to stochastic models with decentralized control.
A widely known array of work uses disk graph-based methods to coordinate multiple robots while maintaining the overall network connectivity as the nodes move \cite{ji2007distributed, notarstefano2006maintaining, yang2010decentralized}. 
Further, disk-based channel models have also been used for coverage controlling in \cite{notarstefano2006maintaining, yang2010decentralized, paraskevas2016distributed, etancelin2019dacyclem}. 
However, most of them overlook the volatility of the network topology caused by the stochasticity of wireless signals.
Additionally, the ``disk" assumption enforces excessively restrictive control on the robots to maintain the local connectivity, sacrificing the coverage gain.

In \cite{schwager2009decentralized, kantaros2016distributed} authors present optimization-based decentralized coverage control for networked robot teams. 
The former approach mainly relies on a static coverage function and fails to adjust to dynamic ROIs.
\rebutf{In the latter, Kantaros et al. proposes optimizing similar auxiliary objectives to ours, considering message-routing and fixed UEs.
However, our work differs in its ability to cater to both fixed and moving UEs while eliminating the need to incorporate UE positions into the optimization problem explicitly.
By using the team abstraction proposed in \cite{belta2004abstraction}, we make our approach highly scalable in the number of UEs.}
In \cite{mox2020mobile, yan2012robotic} the authors combine the stochastic channel models to account for the  \textit{wireless fading} effect in point-to-point communication to find the optimum router configurations under different routing algorithms.

\subsection{Graphical Games}
Graphical games reflect the notion that a multi-player game can be succinctly represented by a graph, and a player's payoff only depends on its neighbors' actions \cite{kearns2001graphical}.
In \cite{daskalakis2006computing, kakade2003correlated, ortiz2020correlated} the authors established the interplay between games' solutions and probabilistic graphical models.
Although the notion of graphical games resembles that of collective dynamics-based approaches, only a few attempts have been made to employ the framework in swarm coordination, despite the success gained by the latter.
In our previous work \cite{9560899}, we presented an MRF-based approach to steer a robot swarm to a flocking consensus, yet the theoretical guarantees and the connection to graphical games were missing.


In \cite{ai2008optimality, paraskevas2016distributed} the authors proposed graphical game-theoretic methods for distributed mobile sensor coverage by conserving energy; however, the works assume stationary neighborhoods while overlooking the robot dynamics and network limitations.

\section{Preliminaries}
\subsection{Graphical Game Theory}
Consider a game involving $n$ players and $\mathcal{A}_i$ define the set of actions available to any player $i \in \{1, \dots, n\}$.
Let 
$x = (x_1, \dots, x_n)$ denote the joint action profile of the players.
We allow the players to play mixed strategies, and the probability that $i$ is playing the action $x_i$ is denoted by the mixed strategy $Q_i(x_i)$.
Further, $-i$ denotes the set of all players but $i$, $(x'_i, x_{-i})$ denotes an alternative action profile where $i$ plays $x'_i$ instead of $x_i$ while $x_{-i}$ remains the same.
We define $Q(x)$ as an arbitrary joint probability distribution over the action profile $x$ with mixed strategies $Q_i(x_i)$ as the marginals.

A graphical game $\Gamma$ consists of the tuple $\langle \mathcal{G}, \mathbf{M} \rangle$ where $\mathcal{G}$ defines a graph whose vertices correspond to the players, and $\mathbf{M}$ represents the set of payoff functions.
\rebut{The graphical game-theory significantly reduce one's interacting agents count from $n$ to it's local neighborhood size \cite{kearns2001graphical}.}
For some player $i$, $\mathbf{M}_i \in \mathbf{M}$, $\mathbf{M}_i : \mathcal{A}_i \xrightarrow[]{} \mathbb{R}$.
Following the definition of expectation, we define the \textit{expected utility} as follows.
\begin{definition}
The expected utility of player $i$ under a joint probability $Q$ is
\begin{equation}
    \mathbb{E}_{Q}\big[\mathbf{M}_i(x)\big] = \sum_{x} Q(x_i, x_{-i})\mathbf{M}_i(x_i,x_{-i}).
\end{equation}
\end{definition}
\begin{definition}
The correlated equilibrium (CE) of a graphical game is a joint distribution $Q$ over the associated undirected graphical model, under which no player has a unilateral incentive to deviate. Thus, for $x_i, x'_i \in \mathcal{A}_i$, $\forall i$, and $x_i \neq x'_i$,
\begin{equation*}
    \mathbb{E}_{Q}\big[\mathbf{M}_i(x_{-i}, x_i)\big] \geq \mathbb{E}_{Q}\big[\mathbf{M}_i(x_{-i}, x'_i)\big].
\end{equation*}
\end{definition}

\begin{definition}
A mixed strategy Nash equilibrium (MSNE) is a special case of CE, where the joint probability is a product distribution of the marginals. Thus, $Q(X)$ $ = \prod_iQ_i(X_i)$ \cite{ortiz2020correlated}.
\end{definition}




\subsection{Variational Energy Functional}
The VI casts the inference problem over an MRF as a convex optimization problem and approximates the posterior distribution much more efficiently in contrast to exact inference methods \cite{koller2009probabilistic}.
Let $\mathcal{G} = (\mathcal{V}, \mathcal{E})$ be an MRF, where $\mathcal{V}$ and $\mathcal{E}$ are the set of vertices and edges, and $\mathcal{V}$ consists of a set of discrete RVs $\{X_1, \dots, X_{n}\}$.
\begin{definition}
The joint probability distribution over an MRF is often represented as a \textit{Gibbs} distribution
\begin{equation}
    p(X = x) = \frac{1}{Z}  \prod_{c \in \mathcal{C}} \phi_c(x_c),
    \label{eq:mrf_def}
\end{equation}
where, $\phi_c$ is a factor potential function associated with some clique $c \in \mathcal{C}$ of $\mathcal{G}$, $\phi_c: X^{|c|} \xrightarrow{} \mathbb{R}^+$, and $Z$ is the partition function to normalize the distribution, where $Z = \sum_{x_c} \prod_{c \in \mathcal{C}} \phi_c(x_c) $. 
\end{definition}
\begin{remark}
For any $\phi_c = \exp \{\varepsilon (x_c) \}$, $p(x)$ defines an exponential family distribution, where $\varepsilon: X^{|c|} \xrightarrow{} \mathbb{R}$ is some function that maps the clique $c$ to a real number.
\end{remark}
Let $Q$ and $P$ denote the approximating and the true posterior distributions in VI.
We consider the $I$-projection of Kullback-Leiber (KL)-divergence between the two distributions,
\begin{equation*}
    D(Q||P_\mathcal{C}) = \mathbb{E}_{Q} \Big[\ln\frac{Q(X)}{P_{\mathcal{C}}(X)} \Big],
\end{equation*}
where, $P_\mathcal{C}$ is the probability distribution over the set of cliques $\mathcal{C}$.
From \eqref{eq:mrf_def} and,
$\mathbb{H}_Q(X) = -\mathbb{E}_Q[\ln Q(X)]$;
\begin{equation*}
    D(Q||P_\mathcal{C}) = -\mathbb{H}_Q(X) - \mathbb{E}_Q\Big[\sum_{c \in \mathcal{C}} \ln \phi_c(x_c) \Big]  + \mathbb{E}_Q [ \ln Z ],
\end{equation*}
\begin{equation}
    D(Q||P_\mathcal{C}) = -F[\tilde{P}_\mathcal{C}, Q] + \ln Z,
    \label{eq:variational_energy}
\end{equation}
for $\tilde{P}_\mathcal{C}(x) = \prod_{c \in \mathcal{C}} \phi_c(x_c)$.
We identify $F[\tilde{P}_\mathcal{C}, Q]$ as the \textit{variational energy functional}. 
From \eqref{eq:variational_energy}, maximum $F$ gives the minimum KL divergence between the approximating and the true posteriors.
In this work, we subsume the two auxiliary objectives into factor potentials associated with cliques of the underlying graphical model.

\begin{figure}[t]
\centering
\subfigure[]
 	{\includegraphics[width = 0.242\textwidth, trim={0.25cm 0.15cm 1.03cm 0.1cm}, clip]{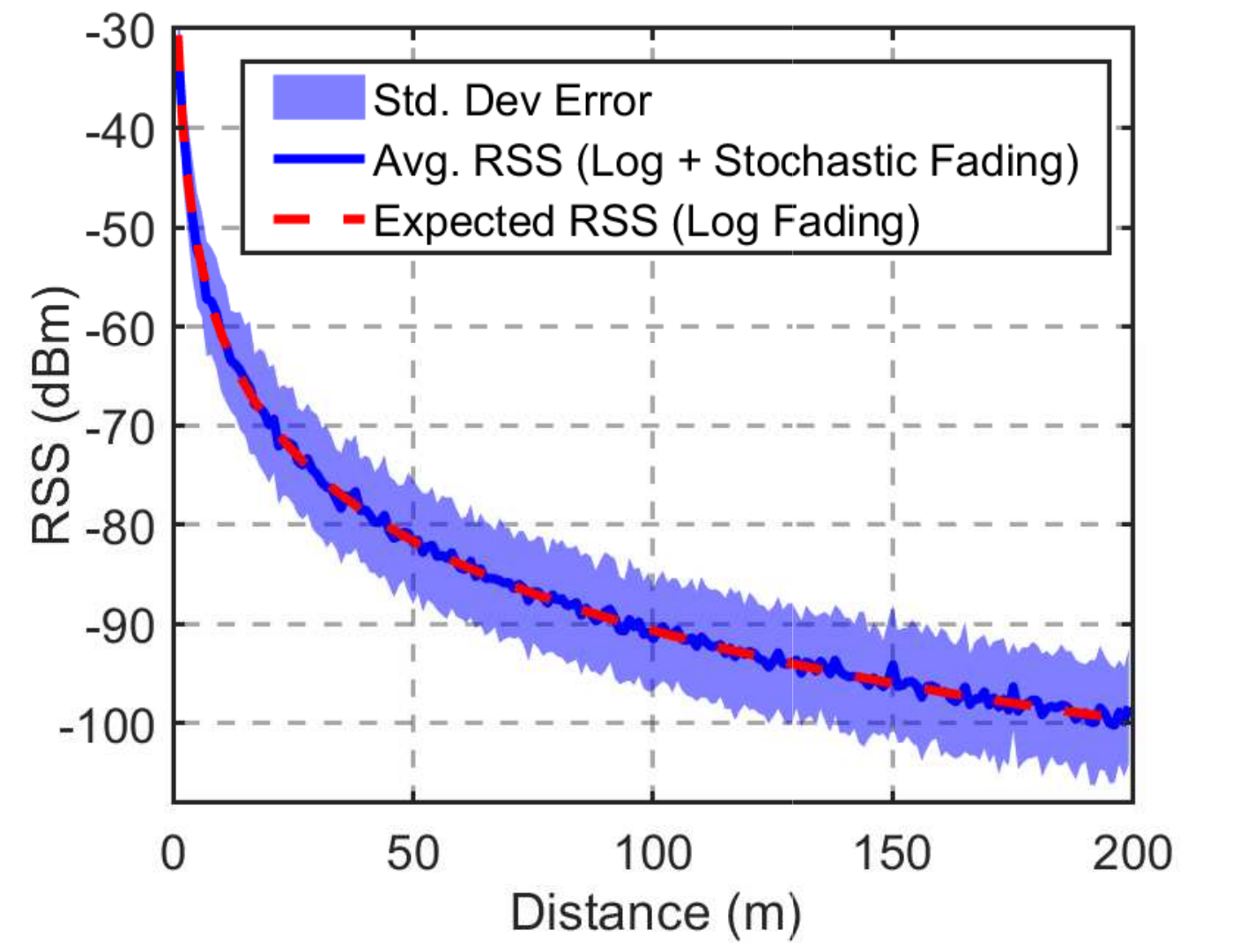}}
\subfigure[]
 	{\includegraphics[width = 0.24\textwidth, trim={0.7cm 0.15cm 1cm 0.1cm}, clip]{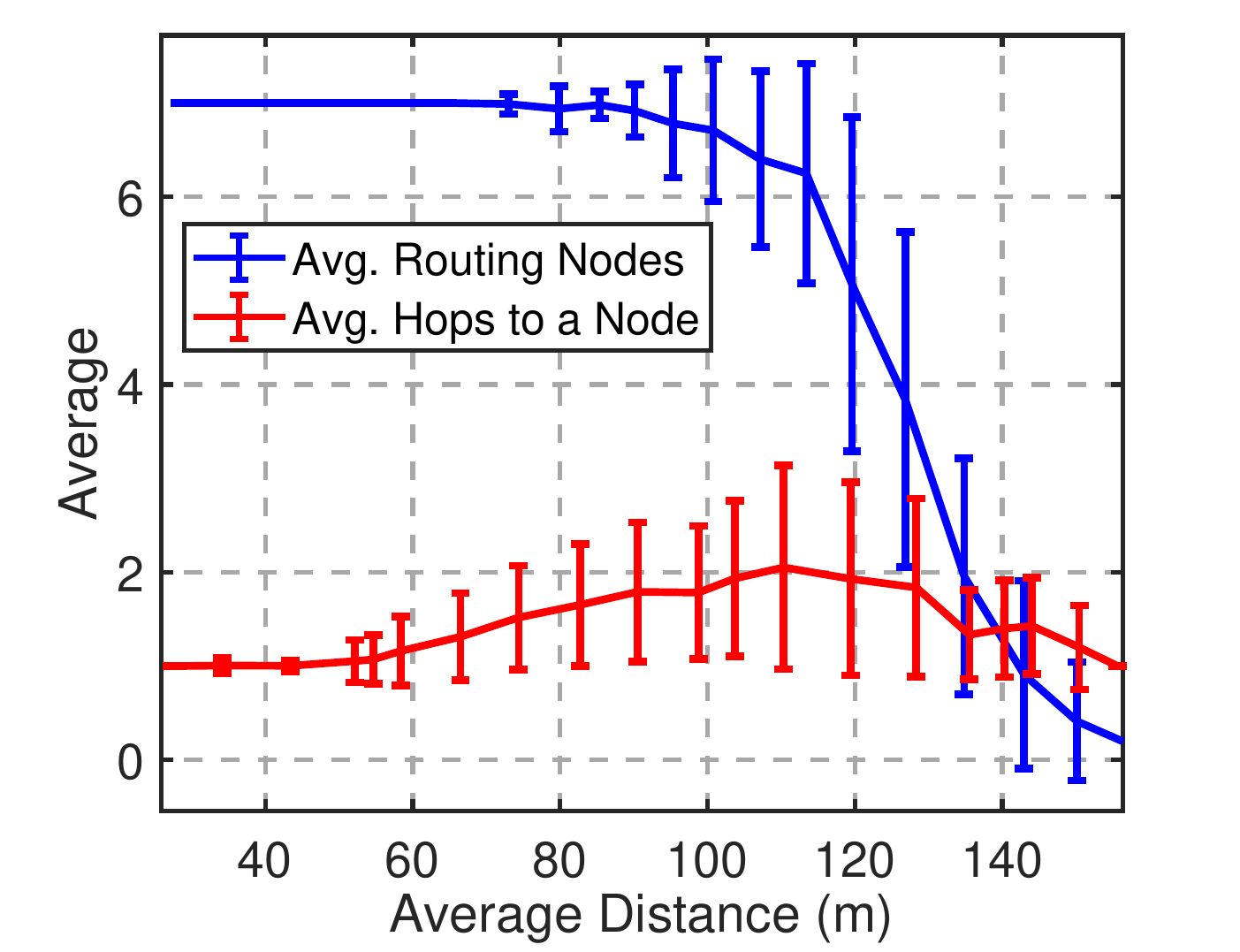}}
\caption{(a) The change of RSS, (b) the number of communicating nodes (blue) and the average hop-count (red) against the physical distance. The $\inf$ hops have been ignored when the link disappears for the purpose of plotting. The parameters used are $n$=3, $\mathcal{F}$=$32(dBm)^2$, $T_0 = 16.02$ dBm. }
 \label{fig:rss}
 \vspace{-10pt}
\end{figure}
\section{Problem Formulation}
To accommodate a wide range of applications, we consider providing wireless network coverage to a team of mobile UEs using a robot (UAV) team.
Note that we only control the robots while the UEs move arbitrarily over the ROI.
Our formulation can be further generalized to other coverage control problems such as surveillance, where the robot teams need to maintain the local connectivity while patrolling over an ROI.
\rebutf{To acco\-mmodate the cooperative nature of the game, we assume that all the robots are interested in maximizing the coverage and thus, share the same payoff function.}

First, we define a \textit{concentration ellipsoid} $\mathcal{R}$ to represent the ROI to capture the UE distribution, following \cite{belta2004abstraction}.
Such an abstraction allows us to easily scale, and to quantify the total wireless coverage independent of the UE team size.
Let $\Sigma$ define the covariance matrix associated with $\mathcal{R}$.

Considering the homogeneity in the robots, we populate the action space $\mathcal{A}_i$ by discretizing the dynamically feasible control action space of a robot.
\rebut{The state equation of a robot $i$ can be written as $\dot{y}_i = Ay_i + Bx_i$, where $A$, $B$ are positive semi-definite matrices and $x_i \in \mathcal{A}_i$.
Let $\dot{y}_i^{pos} \in \mathbb{R}^3$ be the position of $i$ as extracted from the state vector $\dot{y}_i$.
}

\rebutf{We consider that each stage of the communication-aware coverage game corresponds to a timestep $t$ and is a graphical game $\Gamma_t$ defined on the local neighborhoods $\mathcal{N}_i^t, \forall i$.
The aim of our work is to find a consensus UAV formation for $\mathcal{R}$ in the form of a MSNE for $\Gamma_t$. 
In the following sections, we show that this can be achieved by leveraging VI and iterative stage game optimization.}
\section{Approach}



\subsection{Communication-Aware Cooperative Coverage}
In practice, the Recieved Signal Strength (RSS) between any two wireless nodes can attenuate for multiple reasons such as path loss, shadowing, and fading.
Thus, in wireless networks it is common to model the channels in a stochastic fashion \cite{fink2011robust}.
In this work, we employ a channel model for the IEEE 802.11a protocol to obtain the expected RSS in point-point communication.
For a joint action set we define $f_{RSS}: \mathcal{A}_i \times \mathcal{A}_j \xrightarrow[]{} \mathbb{R}$ as
\begin{equation}
    f_{RSS}(x_i, x_j) = T_0 - \big\{ L_0 + 10n . \log(d_{ij}) + \mathcal{F}, \big\}
    \label{eq:rss}
\end{equation}
where, $d_{ij} = \norm{\dot{y}_i^{pos} - \dot{y}_j^{pos}}$, $f_{RSS}(d_{ij})$ is the RSS measured in dBm (decibels relative to milliWatt), $T_0$ is the transmission power, $L_0$ is the reference power loss for free space, $n$ is a path loss exponent and $\mathcal{F}$ is a zero mean Gaussian distribution to account for the \textit{fading} effect.
In this work, we used $L_0$ = 46.67 dBm calculated using the $Friis$ model for open spaces.
Similar real-world experiments have also been conducted in \cite{fink2011robust}, which helped us to model the fading in the channel.

Fig. \ref{fig:rss}(a) shows the change in RSS against the inter-robot distance for the channel model.
As in Fig. \ref{fig:rss}(b), we observed the network is densely connected at the beginning, but separated drastically as the nodes moved further away.
Specifically, the \textit{hop-count} between the robots increased with distance, making it much harder to maintain fixed neighborhoods. 
Here the hop-count refers to the number of interim connections between two communicating nodes in the network.
\rebut{In practice, as the ROI expands, the robots require to travel farther in search of better coverage; however, these experimental results signify the need to adjust the neighborhoods as the network topology changes.}

Thus, we define the local neighborhood of any robot $i$ at time $t$, $\mathcal{N}_i^t(k)$ as the list of $k$-hop nodes in the instantaneous network topology.
Therefore, for any $j \neq i$, $\mathcal{N}_i^t(k) = 
\{j \lvert \mathrm{Hops}(i,j) \leq k \}$,
where $\mathrm{Hops}(i,j)$ is the number of hops to node $j$ according to $i$'s \textit{routing table}.
With the definition of expectation, we obtain the expected RSS $\mathbb{E}[f_{RSS}] = \psi_{R}(x_i, x_j)$ resulted by selecting the actions $x_i, x_j$ as
\begin{equation}
    \psi_{R}(x_i, x_j) = T_0 - \big\{ L_0 + 10n . \log(d_{ij}) \big\}.
    \label{eq:exp_rss}
\end{equation}

We define function $\psi_{C}: \mathcal{A}_i \xrightarrow[]{} \mathbb{R}$ to quantify the coverage provided by the neighborhood $\mathcal{N}_i$.
We define the communication-aware coverage as the expected ``cooperative RSS field'' for the ellipsoidal $\mathcal{R}$.
Therefore,
\begin{equation}
    \psi_{C}(x_{\mathcal{N}}, \mathcal{R}) = \sum_{r \in \mathcal{R}} \max \big\{ \psi_{R}(x_i, r), \psi_{R}(x_{-i}, r)  \big\} p(r),
    \label{eq:coverage}
\end{equation}
where $p(r)$ is the probability of a point $r \in \mathcal{R}$ in the ROI.
In our work, $p(r)$ denotes the probability of having a UE at $r$ under the distribution characterized by $\Sigma$.
Here, $\psi_{R}(x_{-i},r)$ denotes the coverage imparted on $r$ by $i$'s any other neighbor.
When the values assigned to the local neighborhood of $i$, $x_{-i}$, and the ROI are fixed, we observe that the coverage function only varies with $x_i$ within the time interval.
Therefore, we denote the coverage function associated with $\mathcal{N}_i$ as $\psi_C(x_i)$.
This formulation allows us to construct a more realistic \textit{cooperative coverage field} for the neighborhood, as UEs often select the wireless node with the highest signal strength to connect in practice.


Note that for $r$ over large distances, $\psi_{C}(x_i)$ is stationary for some robot $j \in \mathcal{N}_{-i}$ when  $\psi_R(x_i, r) \leq \psi_R(x_j, r)$, as any UE at $r$ would connect to $j$ despite the actions of $i$ due to the higher signal strength.
\rebutf{Thus, this introduces a partitioning in the local coverage functions.}
Fig. \ref{fig:mrf}(a) shows a simple network topology for 4 robot nodes and a $k=1$-hop neighborhood for the $i$-th node.
Next, we use this coverage model to introduce the \textit{sufficient statistics} for MRF and the payoff functions.
\subsection{Payoff Function for the Stage Graphical Game}

Following the communication topology and the neighborhood parameter $k$, we define the graph $\mathcal{G}_t = (\mathcal{V}_t, \mathcal{E}_t)$ for the stage game $\Gamma_t$.
The set of vertices $\mathcal{V}_t = \{X_1, \dots, X_n \}$ comprises the random variables for each robot node.
The set of edges $\mathcal{E}_t$ contains an element $(i,j)$ if and only if $i$ and $j$ satisfies the neighborhood condition under $k$. 
Thus, $\mathcal{E}_t = \{ (i,j) | j \in \mathcal{N}_i^t(k),  \forall i \}$.
Hereafter, we ignore the script $t$, as we are interested in finding the equilibrium for a single-stage game.
Also, for a fixed $k$, let $\mathcal{N}_i^t(k) = \mathcal{N}_i^t$.

We consider that the payoff of a robot $i$ relies on the coverage provided by the neighborhood and its expected RSS with the neighbors.
Formally, we define,
\begin{equation}
 \mathbf{M}_i(x) = \alpha_a \psi_C(x_i) + \alpha_b \sum_{j \in \mathcal{N}_{-i}} \psi_{R}(x_i, x_j),
 \label{eq:payoff}
\end{equation}
where $\alpha_a$, $\alpha_b >0$ are two predefined weight parameters.
In this work, they scale the contributions from two factor potentials $\psi_c$, $\psi_p$ in the payoff function proportionately.
The first term forces the robot to move farther to maximize the coverage, and the second term penalizes the robot for selecting the actions that weaken the signal strength.
\rebut{We assume that the \textit{expected RSS} remains roughly unchanged in close proximities over $\mathcal{R}$.
Thus, considering the ROI size and inter-robot distances, we argue that the contradicting effects of the auxiliary objectives can be ignored for sufficiently small intervals of $t$.}
Therefore, a robot's best response action maximizes the coverage as well as the expected RSS with its neighbors.
In this work, we propose solving the game $\Gamma$ by performing posterior inference over an appropriately tailored MRF.
By using an exponential family distribution, we establish an analogy between the equilibrium in the game and the resulting joint distribution.

\subsection{Exponential Family Posterior Distribution for MRF}
We now discuss integrating the payoff function with a probabilistic graphical model defined on the game $\Gamma = (\mathcal{G}, \mathbf{M})$.
We start by modifying and factorizing $\mathcal{G}$ into pairwise and neighborhood subgraphs to complement the communication-aware coverage model defined in \eqref{eq:exp_rss} and \eqref{eq:coverage}. 
Let us first introduce an auxiliary edge $ (j,h)$ between $j$ and $h$ nodes, if $j,h \in \mathcal{N}_i$ and $(j,h) \notin \mathcal{E}$ for all $i$.
In other words, we convert every local neighborhood of $\mathcal{G}$ into a complete subgraph by adding a set of auxiliary edges, $(j,h) \in \mathcal{E}_{Aux}$.
We denote the derived graph as $\mathcal{G}' = (\mathcal{V}, \mathcal{E}')$, where $\mathcal{E}' = \mathcal{E} + \mathcal{E}_{Aux}$.
Fig. \ref{fig:mrf}(b) shows a clique transformed neighborhood subgraph for an initial communication topology.
We define \textit{neighborhood clique} as a clique that is comprised of the nodes of some neighborhood. 
Next, we factorize the derived graph into a set of cliques $\mathcal{C} = \mathcal{E}' \cap \{\mathcal{N}_i | \forall i\}$, such that each clique $c \in \mathcal{C}$ either represents an edge in $\mathcal{E}'$ or a neighborhood clique.
Finally, we associate each clique $c \in \mathcal{C}$ with a factor potential function $\phi_c : X^{|c|} \xrightarrow[]{} \mathbb{R}^+$. 
\begin{figure}[t]
    \centering
\subfigure[]
  	{\includegraphics[width = 0.23\textwidth,trim={0cm 0cm 0cm 0.1cm}, clip]{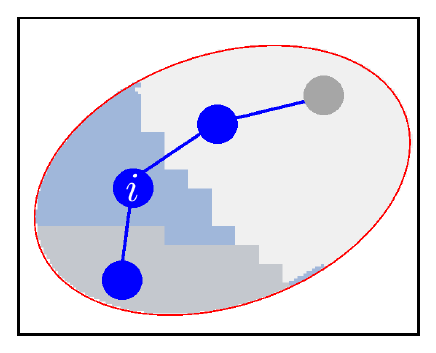}}
\subfigure[]
 	{\includegraphics[width = 0.23\textwidth,trim={0cm 0cm 0cm 0.1cm}, clip]{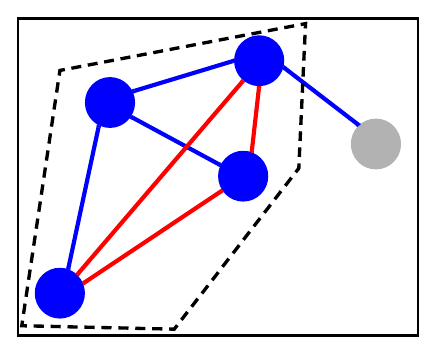}}
\caption{ (a). \rebutf{A visual representation of the local cooperative coverage function of $i$ for a single hop neighborhood. The partitions covered by each neighbor are colored differently.} (b). A communication topology with the newly added auxiliary edges (red) to a neighborhood. The dashed polygon denotes the resulting neighborhood clique. The neighborhood robots and the communication links are shown in blue.}
 \label{fig:mrf}
\vspace{-10pt}
\end{figure}
This formulation allows us to define a joint distribution over the induced MRF $\mathcal{G}'$ using the exponential family similar to \eqref{eq:mrf_def}.
Consider
\begin{equation}
\label{eq:phi}
\phi_c(x_c) = 
\begin{cases}
    \hspace{5pt} \phi_i(x_i) = \exp \big\{ \alpha_{i}. \psi_{C} \big \}  & c \in \{\mathcal{N}_i | \forall i\} \\
    \hspace{5pt}\phi_{ij}(x_i, x_j) = \exp \big \{ \alpha_{ij}. \psi_{R} \big \} & c \in \mathcal{E}',\\
\end{cases}
\end{equation}
where $\phi_i$, $\phi_{ij}$ simply redefine the factor potentials given the clique $c$'s type and $\alpha_{i}$, $\alpha_{ij} \geq 0$ are some weights associated with the cliques.
With this definition, we consider the following joint probability for the MRF $\mathcal{G}'$.
\begin{equation}
    p(x) = \frac{1}{Z} \exp \big\{ \alpha_a\sum_{i \in \mathcal{V}} \psi_C(x_i) + \alpha_b\sum_{ (ij) \in \mathcal{E}} \psi_{R}(x_i, x_j) \big\}.
    \label{eq:mrf}
\end{equation}
\begin{theorem}
The probability over the MRF $\mathcal{G}'$, $p(x)$ defines a linear exponential family with canonical parameters $\alpha$ and sufficient statistics $\psi$.
\end{theorem}
\proof
Define two vectors $\psi$ and $\alpha$ that contain factor potentials and associated weights for cliques $\mathcal{C}$ of $\mathcal{G}'$.
Set each $\alpha_c$ associated with the clique $c$; $\alpha_c = \alpha_a$ for neighborhood cliques, $\alpha_c = \alpha_b$ for pairwise cliques, and $\alpha_c = 0$ for auxiliary pairwise cliques.
Thus, the summations in \eqref{eq:mrf} become $p(x) = \exp \big\{ \langle \alpha, \psi \rangle  - \Lambda \big\}$,
where $\langle \alpha, \psi \rangle$ is the inner product $\alpha$, $\psi$ , and $\Lambda = \ln Z$.
Following \textbf{Remark 1}, this defines a linear exponential family.
\endproof

With \textbf{Remark 1} and \textbf{Theorem 1}, we observe that the posterior $p(x)$ takes the standard form of a joint distribution over an MRF.
Given $p(x)$ and the payoff $\mathbf{M}_i$, we further notice that two functions share a mutual form.
Specifically, the summation terms that pertain to a single robot $i$ inside the exponential of the former resembles  the payoff function.
In the following section, we perform posterior inference over the MRF $\mathcal{G}$ and establish that the resulting probability distribution $Q(x)$ yields a consensus formation of the stage game $\Gamma$.

\subsection{Stage Game Optimization with MFVI}

In \textbf{Theorem 1} we established that the joint posterior $p(x)$ adheres to the form in \eqref{eq:mrf_def}, and in general, $\varepsilon(x_c) = \psi_c(x_c)$ \rebutf{for all $c \in \mathcal{C}$.}
Thus, we start by considering the variational energy functional $F$, associated with \eqref{eq:mrf}.
It is clear from \eqref{eq:variational_energy} that maximizing $F$ results in the minimum KL divergence between the true and the approximating posteriors. 
Here we use MFVI to optimize $F$ and to obtain the approximating posterior.
\rebutf{Specifically, the mean-field assumption considers the posterior distribution as a product of marginals each corres\-ponding to a single RV.}
Therefore, the posterior consists of a set of independently and identically distributed (iid) exponential marginals characterized by their means \cite{wainwright2008graphical}.
This assumption improves the tractability of the inference procedure by restricting the search space to the mean-field family instead of the entire space of distributions.
Using the mean-field assumption, we formally define the optimization problem as 
\begin{subequations}
  \begin{alignat*}{2}
    \mathrm{Finding} \quad & \{Q_i(X_i)\}, \\
    \mathrm{Maximizing } \quad & F[\tilde{P}_\mathcal{C}, Q],  \\
    \mathrm{Subjecting \,to} \quad &Q(X) = \prod_i{Q_i(X_i)}, \notag \\
    & \sum_{x_i}{Q(x_i)} = 1, \forall i = 1 \dots, n.
    \nonumber,
  \end{alignat*}
\end{subequations}
where $Q_i(x_i)$ denotes the marginal probability distribution of $i$ in the posterior.
The goal of the MFVI is to find an update rule for each marginal distribution while keeping the neighboring RVs fixed.
From \eqref{eq:variational_energy}, the variational energy functional takes the form
\begin{equation*}
   F[\tilde{P}_\mathcal{C}, Q] = \mathbb{H}_Q(X) + \mathbb{E}_Q\Big[\sum_{c \in \mathcal{C}} \ln \phi_c(x_c) \Big].
\end{equation*}
We consider the variational energy imparted on $i$ from its neighborhood $\mathcal{N}_i$ under joint $Q$ expressively,
\begin{equation}
    F_i = \mathbb{H}_Q(X_i) +  \mathbb{E}_Q\Big[ \sum_{j \in \mathcal{N}_{-i}} \ln \phi_{ij}(x_i, x_j) + \ln \phi_{i}(x_i)\Big].
    \label{eq:F_i}
\end{equation}
In order to maximize $F_i$, we write the Lagrangian $\mathcal{L}_i$ with the Lagrange multiplier being $\lambda$ as
\begin{equation*}
    \mathcal{L}_i[Q]  = \mathbb{H}_Q(X_i) + \sum_{c \in \mathcal{C} } \mathbb{E}_Q\Big[ \ln \phi_{c} \Big] +  \lambda \Big\{ \sum_{x_i}{Q_i(x_i)} - 1 \Big\}.
\end{equation*}
Next, we differentiate the Lagrangian $\mathcal{L}_i$ w.r.t the marginal $Q_i(x_i)$ and obtain the fixed points corresponding to the maximum.
\begin{equation}
\begin{split}
    \frac{d}{d Q(x)} \mathcal{L}_i[Q] & = - \ln Q(x) - 1 + \sum_{c \in \mathcal{C}} \mathbb{E}_Q \big[\ln \phi_{c} | x_i\big] + \lambda.
\end{split}
\label{eq:derivative}
\end{equation}
During the differentiation step we used the facts that the derivatives of the entropy $\mathbb{H}_Q(X_i)$ and expectation $ \mathbb{E}_Q[\ln \phi_{c}] $ terms w.r.t  $Q(x_i)$ are $-\ln Q(x_i) - 1$ and $\mathbb{E}_Q[\ln \phi_{c} | x_i]$ respectively.
Here $\mathbb{E}_Q \big[\ln \phi_{c} | x_i\big]$ denotes the conditional expectation of the factor potential $\phi_c$ given the value $x_i$.
Notice that the factor function $\phi_c$ changes according to \eqref{eq:phi} given the clique type. 
Setting the derivative to 0 and taking the exponentials of the both sides yields the update rule
\begin{equation}
    Q_i(x_i) = \frac{1}{Z_i} \exp \Big\{  \mathbb{E}_Q \Big[ \sum_{j \in \mathcal{N}_{-i}}\ln \phi_{ij} | x_i + \ln \phi_{i} \Big] \Big\},
    \label{eq:update}
\end{equation}
where $Z_i$ is a typical exponential family normalization constant, as introduced in \textbf{Definition 4}, and the Lagrange multiplier $\lambda$ gets dropped in the normalization.
The resulting update rule is also known as the \textit{coordinate ascent mean-field approximation}.
Further, the redefinition of $\phi_c$ in \eqref{eq:phi} ensures that $\ln \phi_c(x_c)$ exists and reflects the coverage model.
\begin{algorithm}
\SetAlgoLined
 Populate $\mathcal{N}_i, \forall i$ using the routing table for $i$\\
 Construct $\mathcal{G}'$ \\
\For{$i \leftarrow 1 \dots r \dots n$} {
 $Q_i(X_i) \leftarrow \frac{1}{Z_i} \{\phi_i(x_i)\}$ \\
}
 \While{$Q_{old}(X) \neq Q(X)$}{
 $Q_{old}(X) = Q(X)$ \\
 Choose $X_i$ from $\{X_1, \dots , X_n\}$ \\
  $\hat{Q}_i(x_i, x_j) \leftarrow \sum_{\substack{j \in \mathcal{N}_{-i}}} \Big\{ \sum_{x_j} \ln \phi_{ij}(x_i, x_j)  \Big\}$  \\
  $\tilde{Q}_i(x_i) \leftarrow  \sum_{x_i} \ln \phi_i(x_i) $ \\
  $Q_i(x_i)  \leftarrow \frac{1}{Z_i} \exp \{\hat{Q}_i(x_i, x_j) + \tilde{Q}_i(x_i) \} $ \\
  $Q(X) \leftarrow \prod_{i}Q_i(X_i)$
 }
Execute $\argmax_{x_i} Q_i(X_i)$ on robot $i$
 \caption{Mean-Field Stage Game Optimization}
 \label{algo}
\end{algorithm}

The stage game optimization algorithm summarizes the key steps of MFVI into an iterative procedure to solve the graphical game $\Gamma = \langle \mathcal{G}, \mathbf{M} \rangle $ by performing posterior inference on MRF $\mathcal{G}'$.
Specifically, we optimize each RV of the MRF to calculate the marginal probabilities $Q_i(X_i)$ while fixing the neighboring RVs. 
In the graphical game theoretical paradigm, this is equivalent to calculating the mixed strategies profile of each player given the neighborhood.
In the next section, we next show that a resulting posterior probability distribution $Q$ induces an equilibrium in the graphical game $\Gamma$.

\subsection{Equilibrium in the Stage Game}

Consider the modified stage game $\Gamma = \langle \mathcal{G}, \mathbf{M} \rangle$ where $\mathbf{M}_i \in \mathbf{M}$.
In this section, we discuss the interplay between the equilibrium solutions of $\Gamma$ and the posterior probability of the factorized MRF induced by $\mathcal{G}'$.
\begin{theorem}
The joint posterior distribution $Q^*(x)$ over the MRF induced by $\mathcal{G}'$ results a CE for the stage game.
\end{theorem}
\begin{proof}
According to \eqref{eq:F_i} and \eqref{eq:derivative}, each marginal $Q_i(X_i)$ that comprises $Q^{*}(X)$ defines a local maximum of the energy functional $F_i$.
Therefore, for some actions $x_i, x'_i \in \mathcal{A}_i$, under $Q^*(X)$ the marginals $Q_i(x_i) \geq Q_i(x'_i)$, and thus,
\begin{equation}
   \ln Q_i(x_i) \geq \ln Q_i(x'_i).
   \label{eq:qi}
\end{equation}
Now consider an action profile $x=(x_i, x_{-i})$ where $x_{-i}$ represents the actions of $\mathcal{N}_{-i}$ and $x_j \in x_{-i}$ for some $j\in \mathcal{N}_{-i}$.
From \eqref{eq:update},
\begin{equation*}
   \ln Q_i(x_i) \propto \mathbb{E}_{Q^*} \Big[ \sum_{j \in \mathcal{N}_{-i}}\ln \phi_{ij}(x_i, x_j) + \ln \phi_{i}(x_i) \Big].
\end{equation*}
Substituting from the definition \eqref{eq:phi} for $\phi_i, \phi_{ij}$ gives,
\begin{equation*}
   \ln Q_i(x_i) \propto \mathbb{E}_{Q^*} \Big[ \sum_{j \in \mathcal{N}_{-i}} \alpha_b \psi_{R}(x_i, x_j) + \alpha_a \psi_{C}(x_i) \Big].
\end{equation*}
Notice that the summations inside the expectation match the definition of a player's payoff function \eqref{eq:payoff}.
Therefore, substituting from the payoff function,
\begin{equation}
   \ln Q_i(x_i) \propto \mathbb{E}_{Q^*} \Big[\mathbf{M}_i(x_i,x_{-i}) \Big].
   \label{eq:propto}
\end{equation}
From \eqref{eq:qi} and \eqref{eq:propto},
\begin{equation*}
    \mathbb{E}_{Q^{*}}\big[\mathbf{M}_i(x_i,x_{-i})\big] \geq  \mathbb{E}_{Q^{*}}\big[\mathbf{M}_i(x'_i, x_{-i})\big].
\end{equation*}
Thus, according to the definition, $Q^*(X)$ yields a CE of the stage game $\Gamma$.
\end{proof}
According to \textbf{Definition 3}, we observe that the CE resulting from MFVI is indeed an MSNE, due to the product form of $Q^*(X)$.
Therefore, we argue that the posterior inference over MRF $\mathcal{G}'$ results in a consensus formation for the stage game $\Gamma$.
However, recall that the stage game is sub-terminal as the calculated mixed-strategy actions are the controls of the robots' dynamical model.
Therefore, we solve the stage game in iteratively to obtain continuous trajectories and, subsequently, a consensus formation for the communication-aware coverage game.
The distributed configuration of our approach reduces the stage game graph into the neighborhood clique of the derived graph $\mathcal{G}'$ for any robot, rendering the distributed setting much more desirable for controlling large-scale swarms.

\section{Experiments and Results}
\subsection{Experiments Setup}
\rebut{We evaluated the proposed approach using NS-3 and Robot Operating System (ROS) environments in a mobile wireless network scenario, that consists of UAVs and a mobile UE team.\footnote{Find the Mavswarm simulator we used for this work at: \url{https://github.com/malintha/multi\_uav\_simulator}.}
Briefly, NS-3 is a widely employed event simulator to design and implement network models, and we delegate the task of network packet routing to NS-3 by accounting for wireless signal attenuation as the nodes move.
Each UAV was equipped with two wireless interfaces that complied with the channel model to establish the 1) inter-UAV, and 2) UAV-UE links. 
\rebutf{The action set of a UAV consists of acceleration-based discretized control inputs, and we select the actions by optimizing the stage games iteratively to reach a consensus formation.}
The UAVs communicate with their neighbors over the inter-UAV network for optimizing the game. 
The UAV-UE links were only used to measure the coverage RSS of the UEs for evaluation purposes.
}

Further, to calculate the routing paths as the UAV and UE nodes move, we used NS-3's inbuilt \textit{Optimized Link State Routing} (OLSR) algorithm \cite{clausen2003olsr}.
The OLSR algorithm's ability to populate the routing tables by accounting for the communication uncertainties helps us update the local neighborhoods in real-time.
The ROS and NS-3 platforms communicate throu\-gh \textit{Remote Procedure Calls} (RPC) to update the game accordingly.
\rebut{For all the UAVs, we ran the stage game optimization algorithm at a modest frequency $1/t =$ 1Hz in order to account for communication delays.}
We further assume that the UAVs can observe the ROI abstraction, characterized by its mean and the covariance matrix $\Sigma$, by communicating with a base station.
\rebut{We considered a \textit{differentially flat} dynamical model for fixed altitude navigation for the UAVs.
Thus, we uniformly sampled the acceleration space of a UAV within bounds [$-3ms^2, 3ms^2$] to populate $\mathcal{A}_i \subset \mathbb{R}^2, \forall i$ along each X, Y dimension similar to our previous work \cite{9560899}}.
For the computational feasibility, we discretized the ROI $\mathcal{R}$ into $10m \times 10m$ cells.
All the implementations and experiments are conducted using the C++ programming language.

\subsection{Experiments}
\begin{figure}[t]
\centering
\subfigure[]
 	{\includegraphics[width = 0.241\textwidth, trim={0.2cm 0.0cm 0.6cm 0.1cm}, clip]{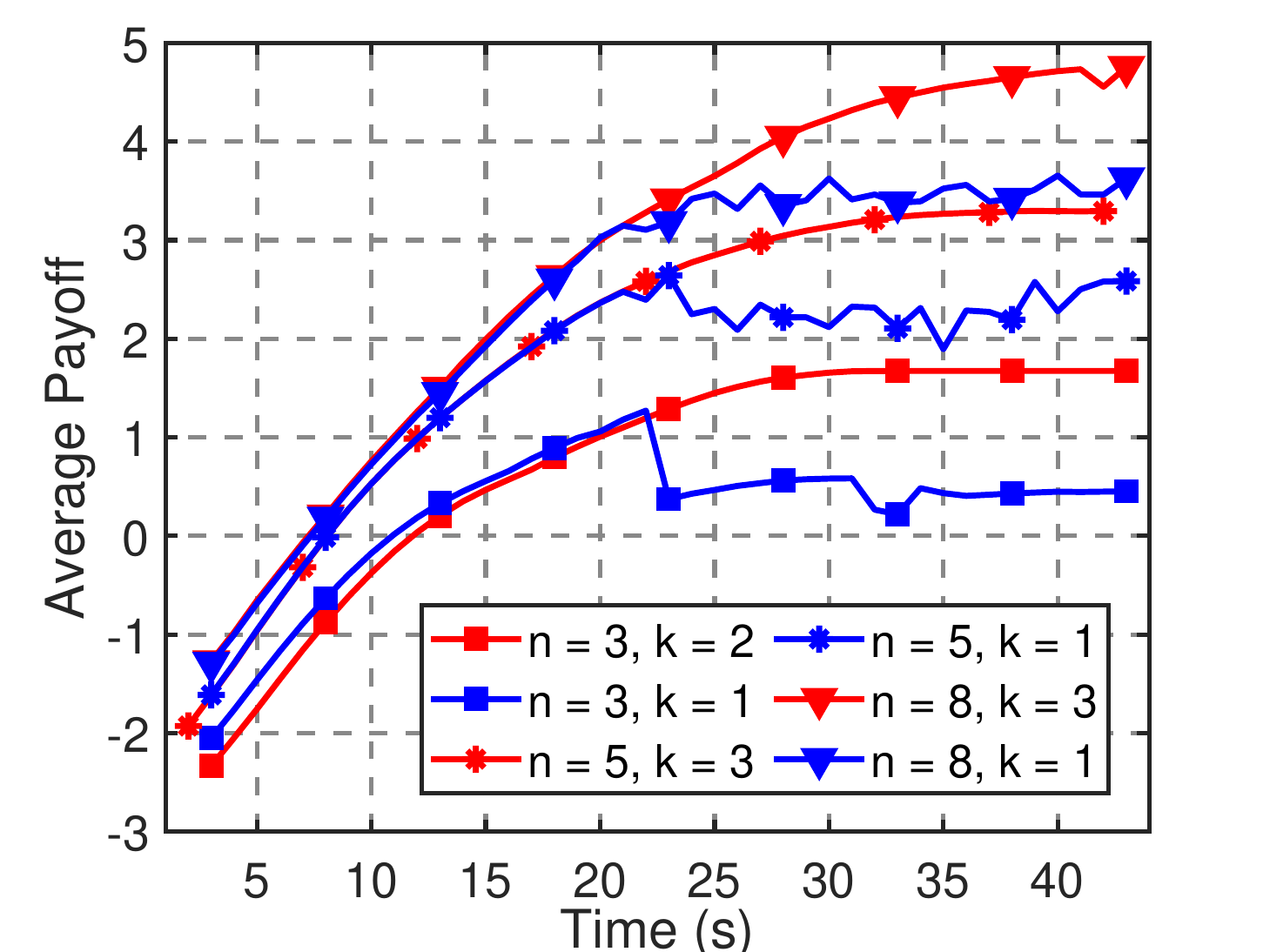}}
\subfigure[]
 	{\includegraphics[width = 0.241\textwidth, trim={0.5cm 0.0cm 0.6cm 0cm}, clip]{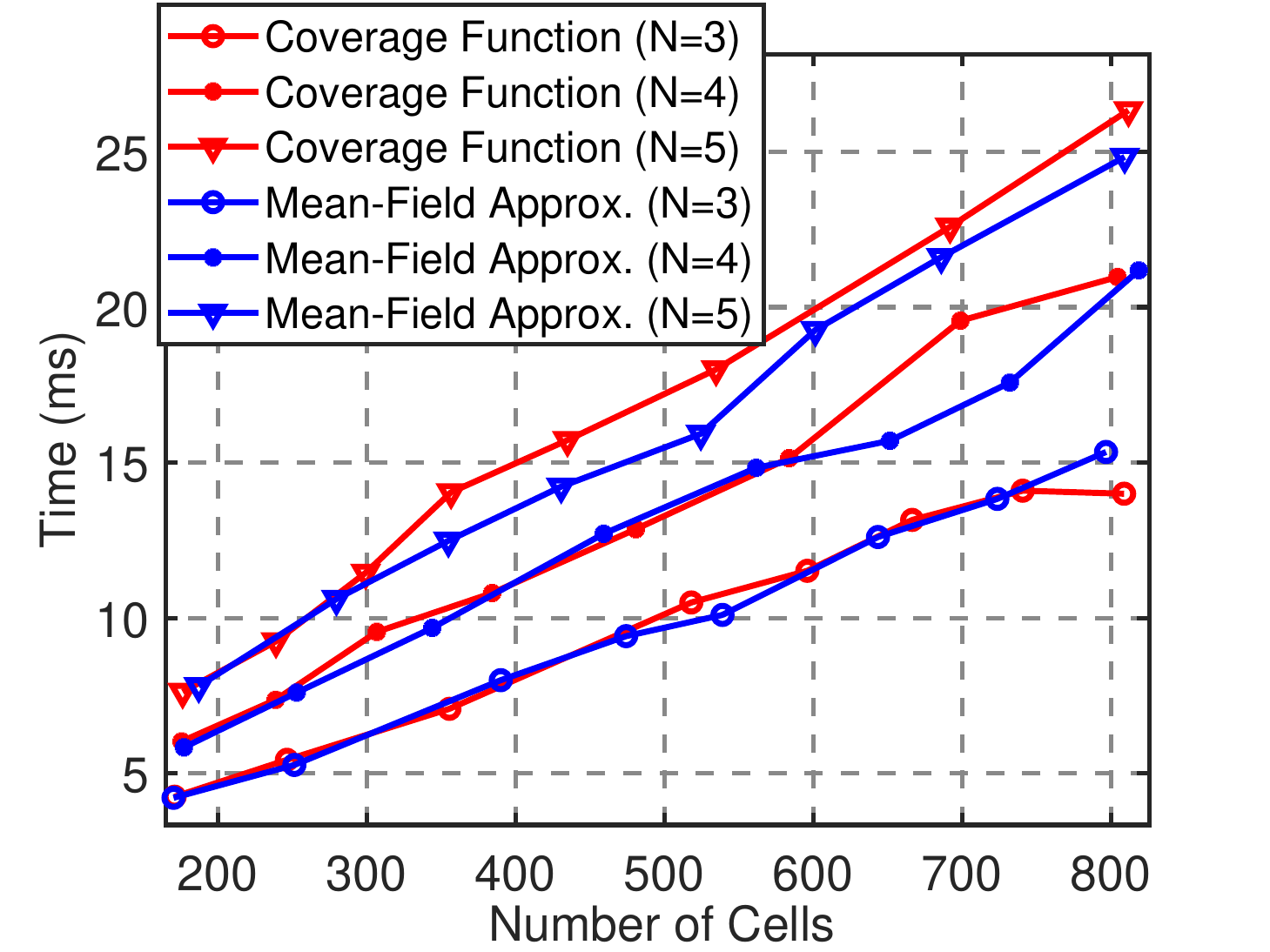}}
\caption{(a). Change of payoff according to Eq. \eqref{eq:payoff} as the UAV swarm navigate over a stationary ROI. The values are plotted by changing the swarm size and the communication hop-count. (b). Computational time to optimize a single stage of the game against the number of neighbors $|\mathcal{N}_i|$ and the size of the ROI. Each cell represents an area of 100$m^2$.}
 \label{fig:hops}
 \vspace{-5pt}
\end{figure}
We evaluated the coverage and computational performance of the proposed approach against the number of UAVs and the size of the ROI.
We accommodated the latter scenario by allowing the UE team to travel between arbitrarily chosen start and goal locations, resulting the ROI's shape and size to change over time (Fig. \ref{fig:coco_images}).
We observed that for small ROIs, $\alpha_b$'s effect is minimal, thus rendering the cooperative coverage to govern the payoff. 
This reduces the tuning effort of the proposed method to a great extent. 
Throughout the experiments we used the parameters $\alpha_a = 1$ and $\alpha_b = 0.001$ to cater to varying ROI sizes.

We first computed the average payoff of a UAV for the stationary ROI case by varying the maximum hop-count and the number of UAV nodes in the swarm as in Fig. \ref{fig:hops}(a).
Even though the single-hop neighborhoods' payoffs overlapped with those of higher-order at close proximities, as the UAVs travel farther, the stability of the formations and the payoffs deteriorated.
This is mainly caused by the densely connected network topology at close proximities; which permits the UAVs to observe the global swarm state with fewer hops.
Therefore, it is understandable that as the network separates, the cooperative coverage deviates from the global value due to the increased locality.
By increasing the allowed maximum hops, we show our approach can result in better payoffs and stable UAV formations; without explicitly aggregating the global swarm state.
\rebutf{Thanks to the adaptive neighborhood property that complements the network dynamics, our approach scales well with the size of the swarm and ROI.}

For the performance evaluations, we used $k=3$ as the allowed communication hops throughout the experiments.
We evaluated computational time for two crucial steps in the game: 1) calculating the cooperative coverage and 2) mean-field approximation stages.
Specifically, we performed the calculations against the number of the neighbors in a game $|\mathcal{N}_i|$, and the size of the discretized ROI.
In the implementation, we calculated the fixed values for any $r\in\mathcal{R}$ in the cooperative coverage function $\psi_C$ beforehand to eliminate redundant calculations and $\max(.)$ comparisons within a single stage of the game.
We observed this to improve the computational efficiency by multitudes, especially optimizing a single stage of the game under 50ms for most cases. 
Fig. \ref{fig:hops}(b) shows the computational time for the two corresponding steps.


Fig. \ref{fig:coco_images} shows the robot trajectories for $n=3, 5$ and $8$ UAVs scenarios.
\rebut{
We defined the ROI using the concentration ellipsoid for 5 UEs distributed around the \textit{origin}.
In each scenario, we simulated the system with stationary and moving ROIs for 40s and 80s intervals, respectively.}
We initialized the UAVs from the origin for all the experiments with a fixed altitude, and maintaining the UEs stationary at first.
Fig. \ref{fig:coco_images}(a-c) show the equilibrium UAV formations and the communication topologies.
\begin{figure}[t]
\centering
\subfigure[]
 	{\includegraphics[width = 0.155\textwidth, trim={0.04cm 0.04cm 0.04cm 0.04cm}, clip]{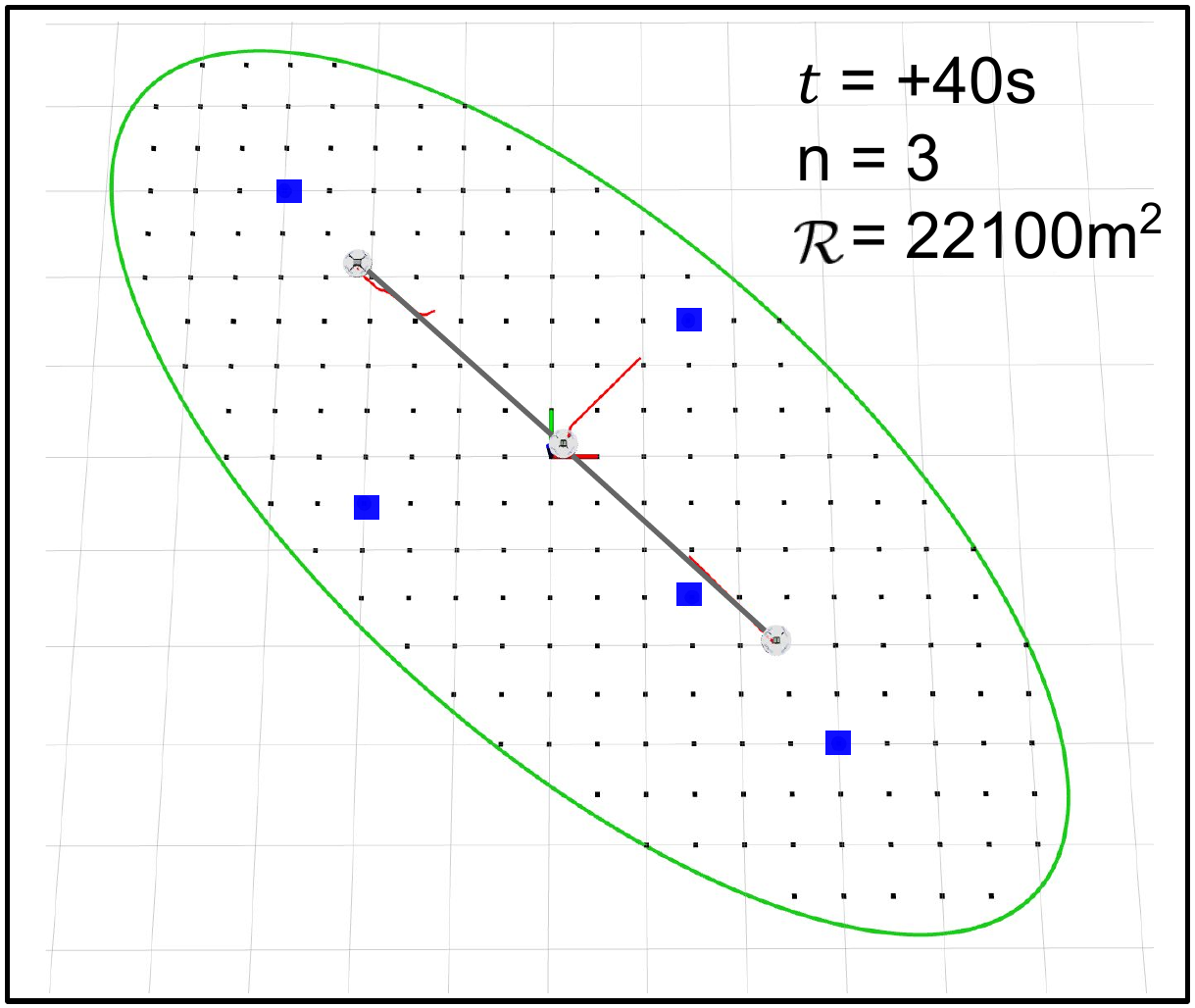}}
\subfigure[]
 	{\includegraphics[width = 0.155\textwidth, trim={0.04cm 0.04cm 0.04cm 0.04cm}, clip]{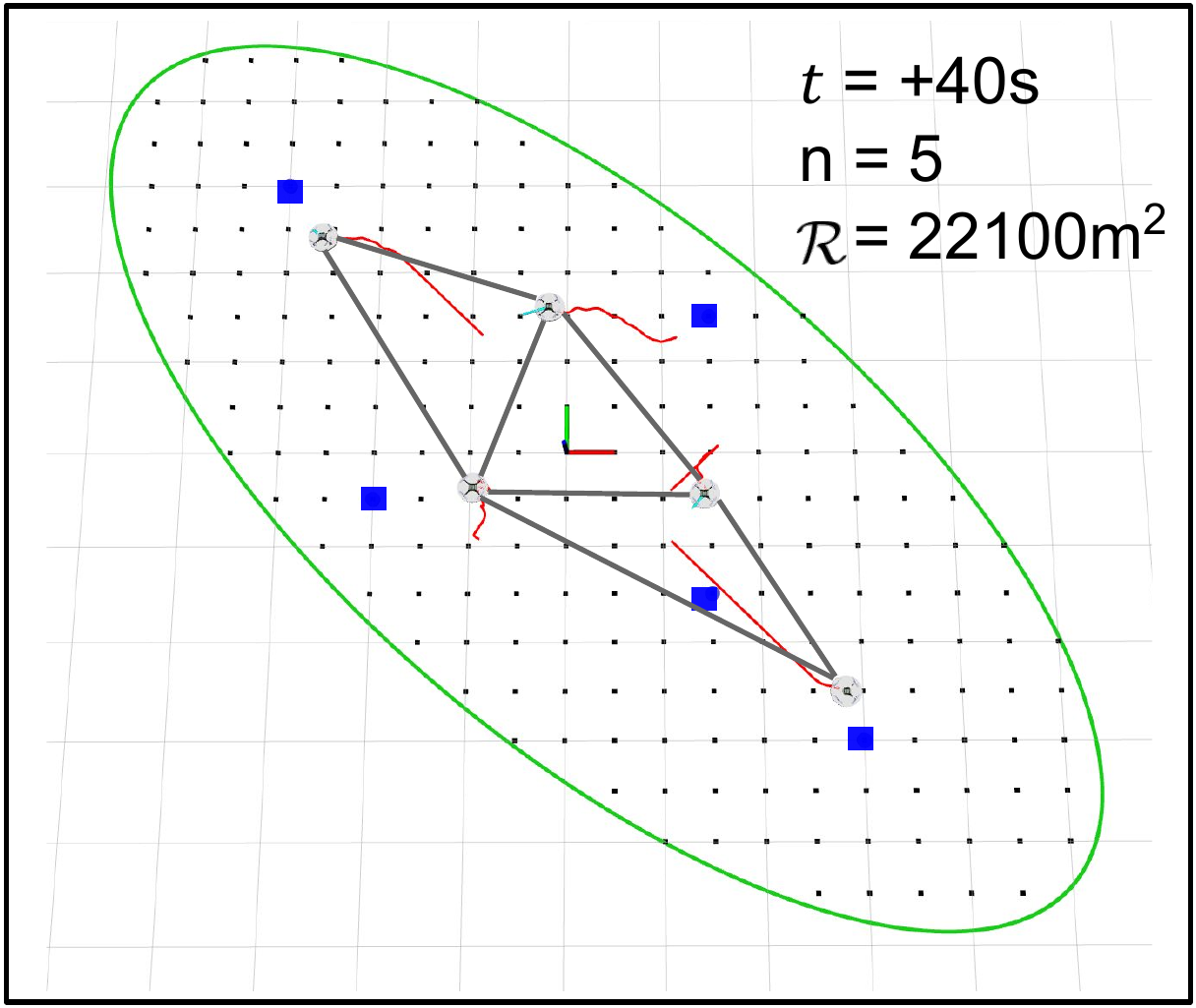}}
\subfigure[]
 	{\includegraphics[width = 0.155\textwidth, trim={0.04cm 0.04cm 0.04cm 0.04cm}, clip]{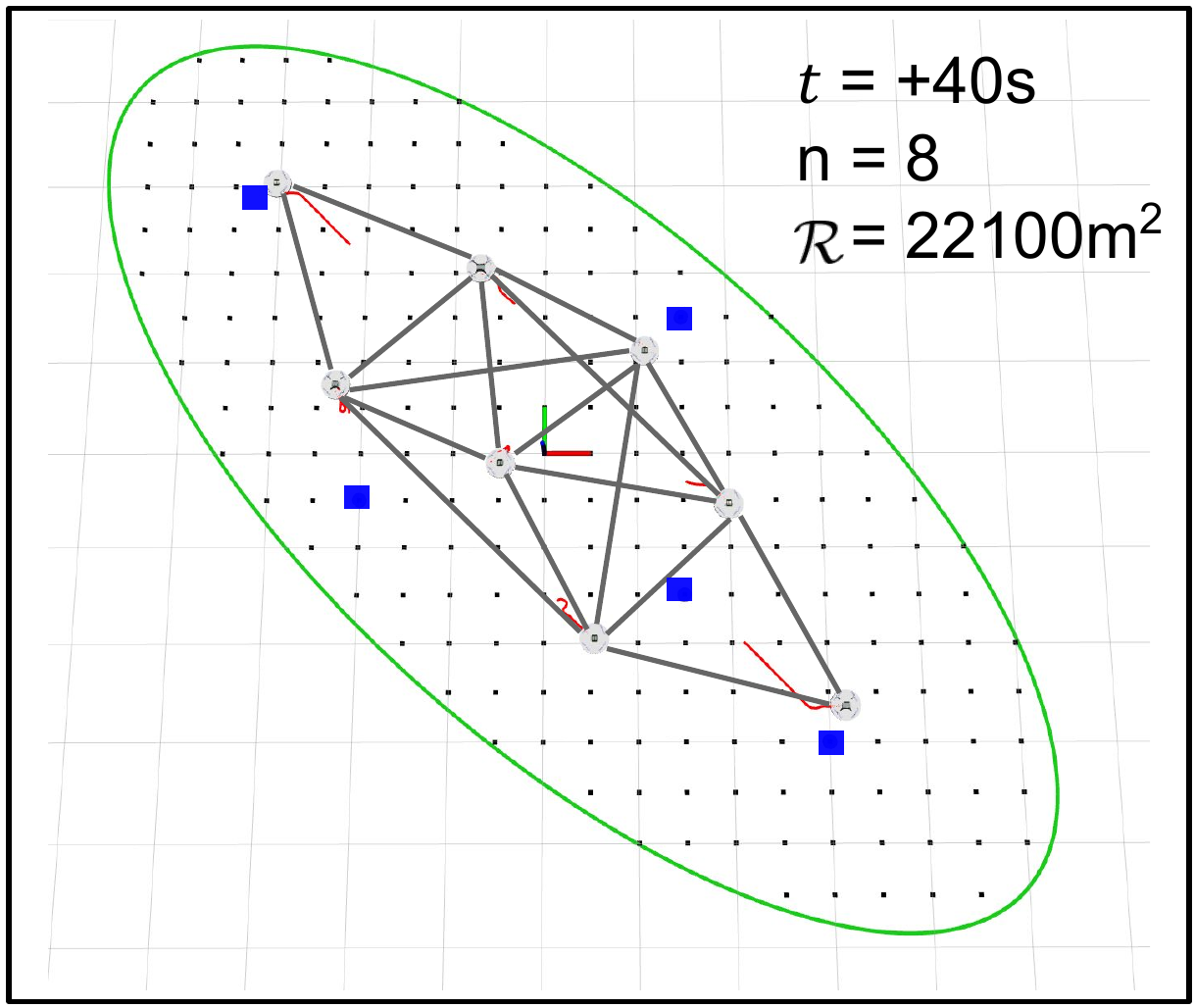}}

\subfigure[]
 	{\includegraphics[width=0.155\textwidth,trim={0.04cm 0.04cm 0.04cm 0.04cm}, clip]{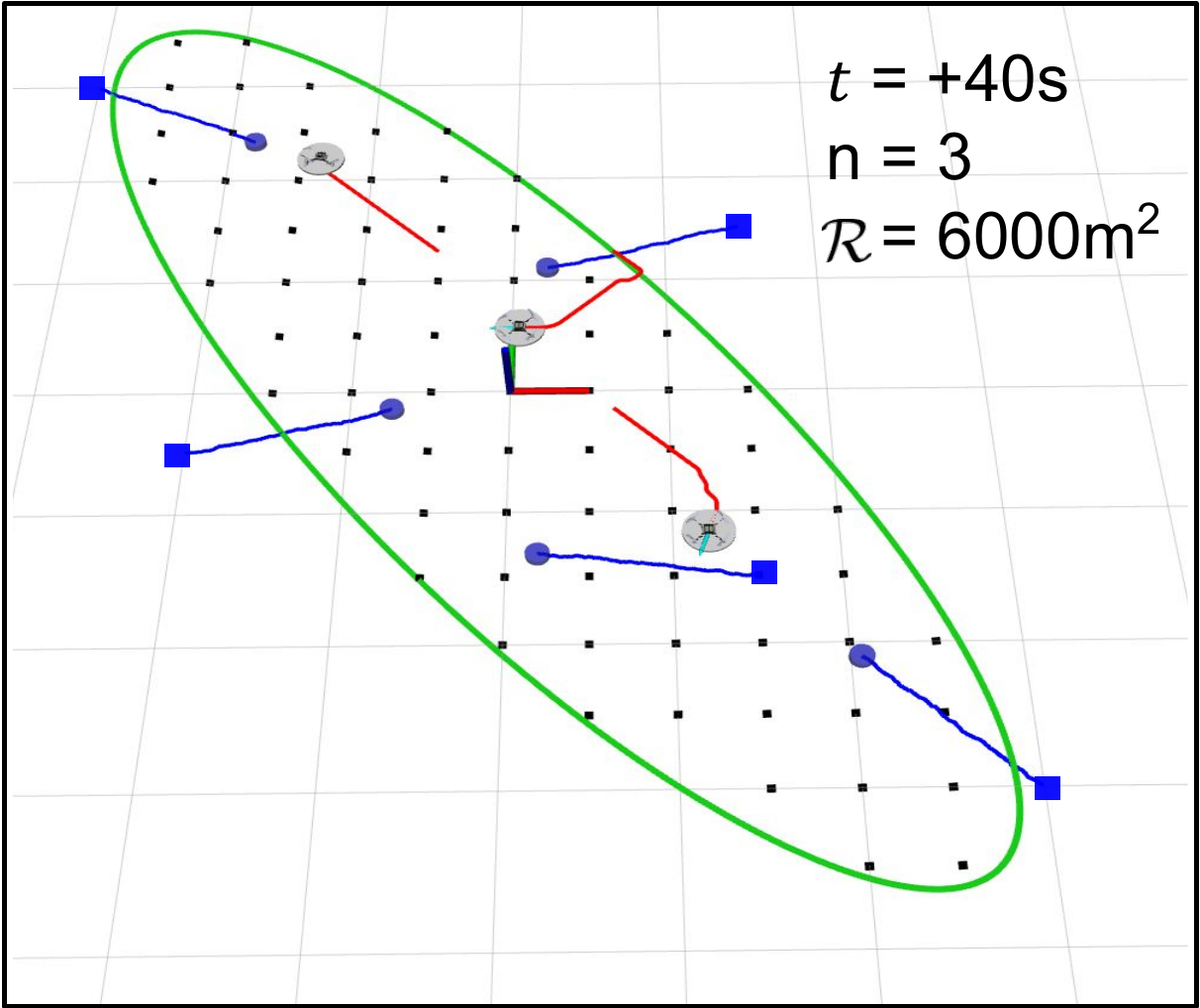}}
\subfigure[]
 	{\includegraphics[width=0.155\textwidth,trim={0.04cm 0.04cm 0.04cm 0.04cm}, clip]{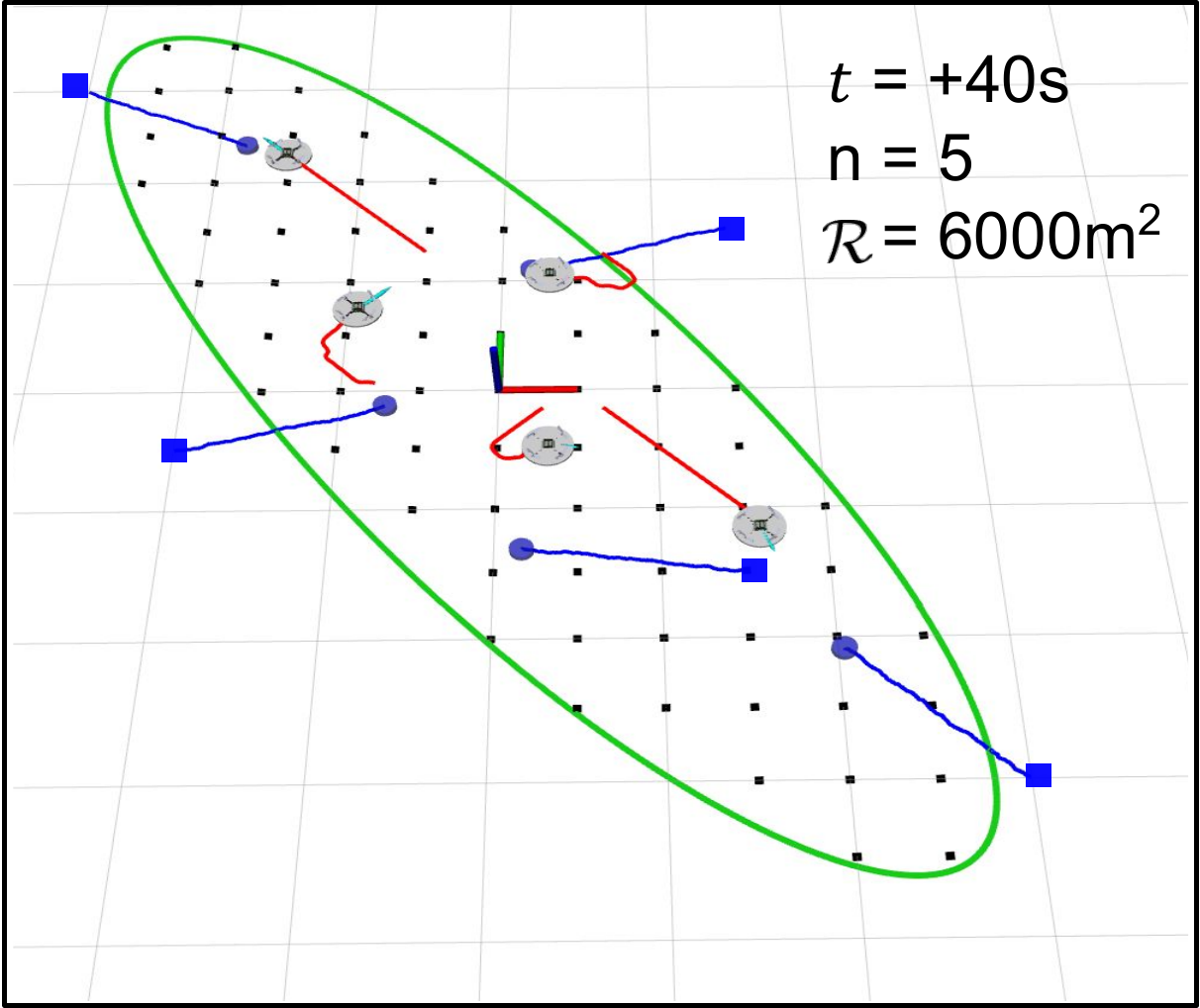}}
\subfigure[]
 	{\includegraphics[width=0.155\textwidth,trim={0.04cm 0.04cm 0.04cm 0.04cm}, clip]{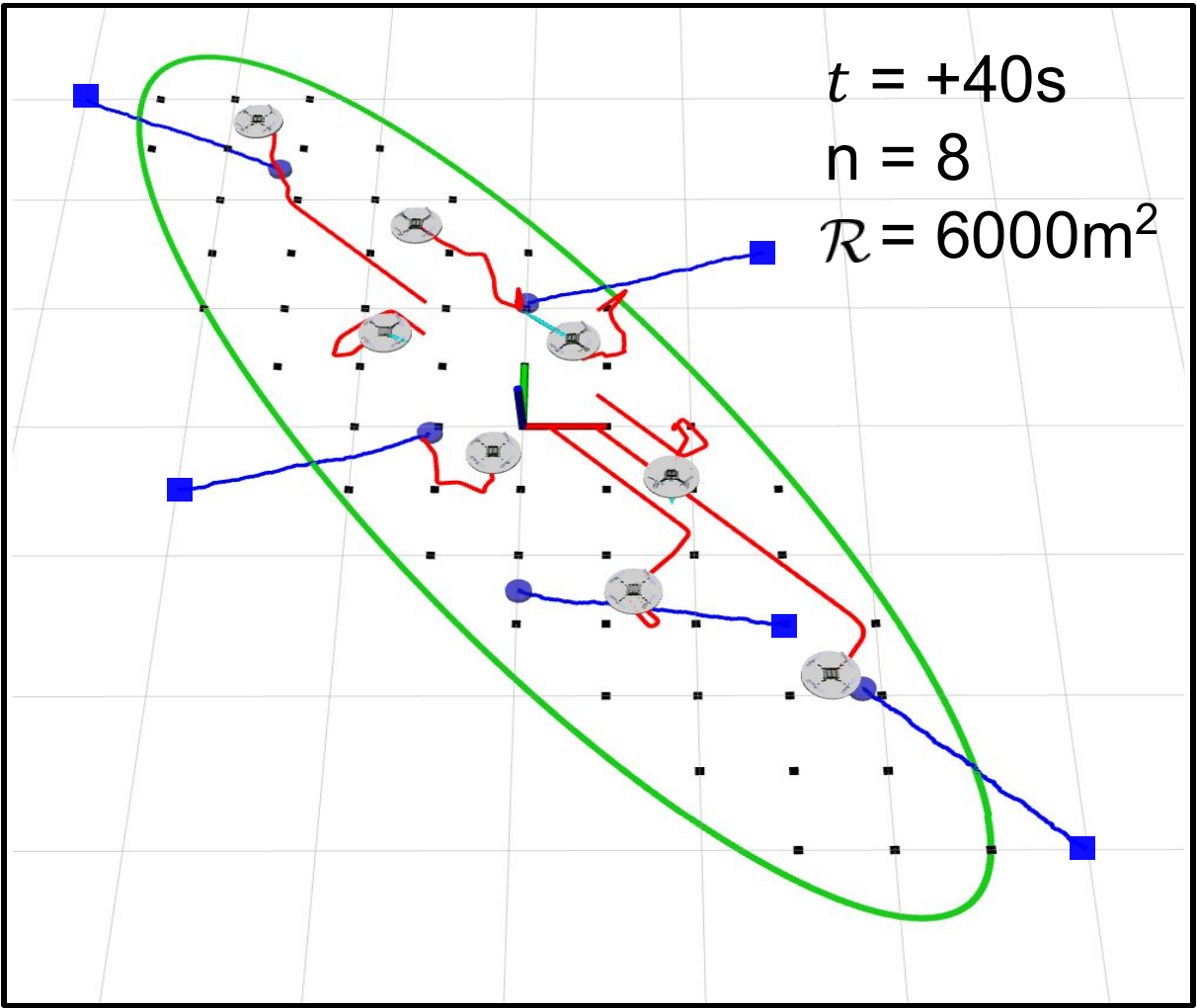}}

\subfigure[]
 	{\includegraphics[width=0.155\textwidth,trim={0.04cm 0.04cm 0.04cm 0.04cm}, clip]{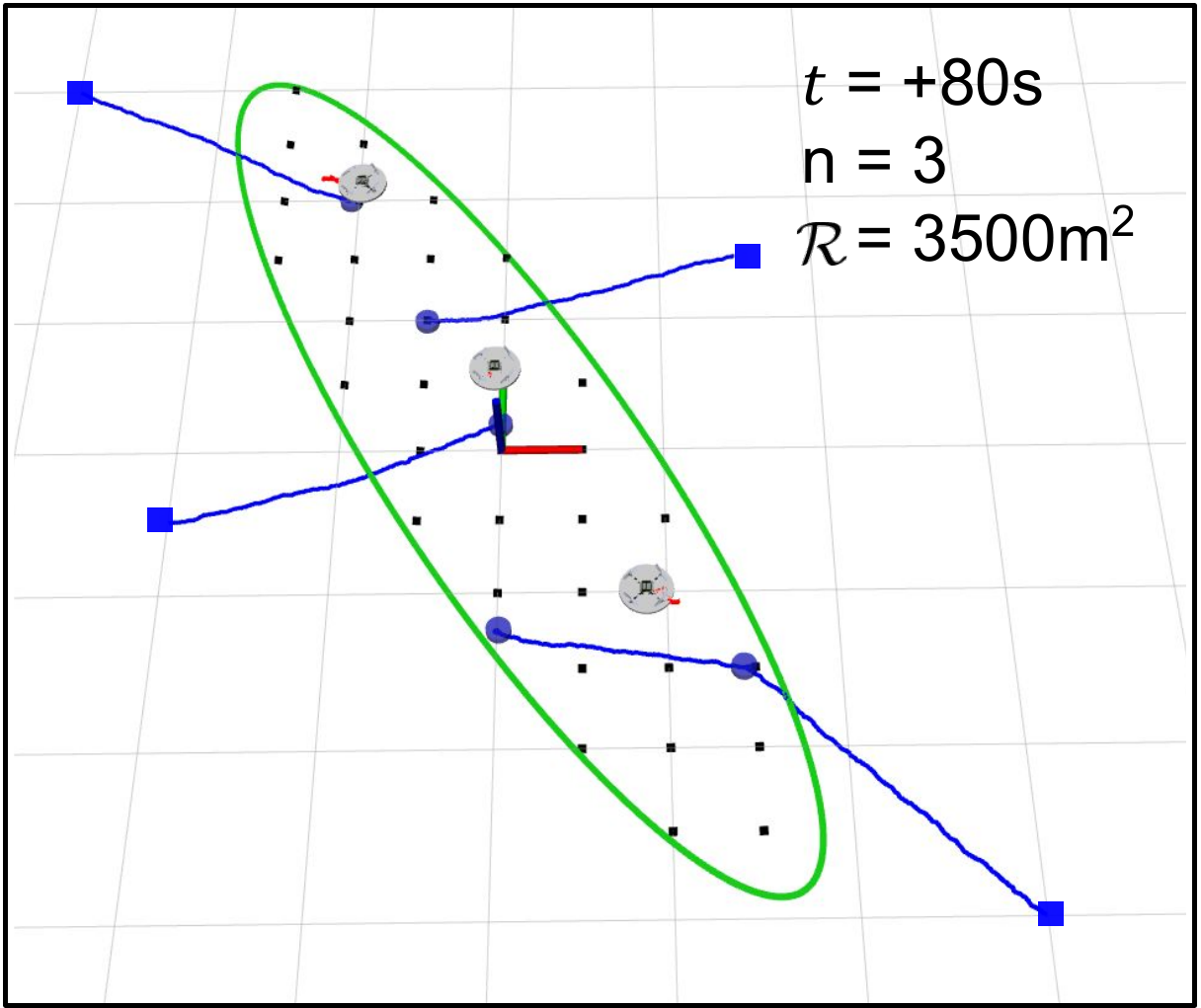}}
\subfigure[]
 	{\includegraphics[width=0.155\textwidth,trim={0.04cm 0.04cm 0.04cm 0.04cm}, clip]{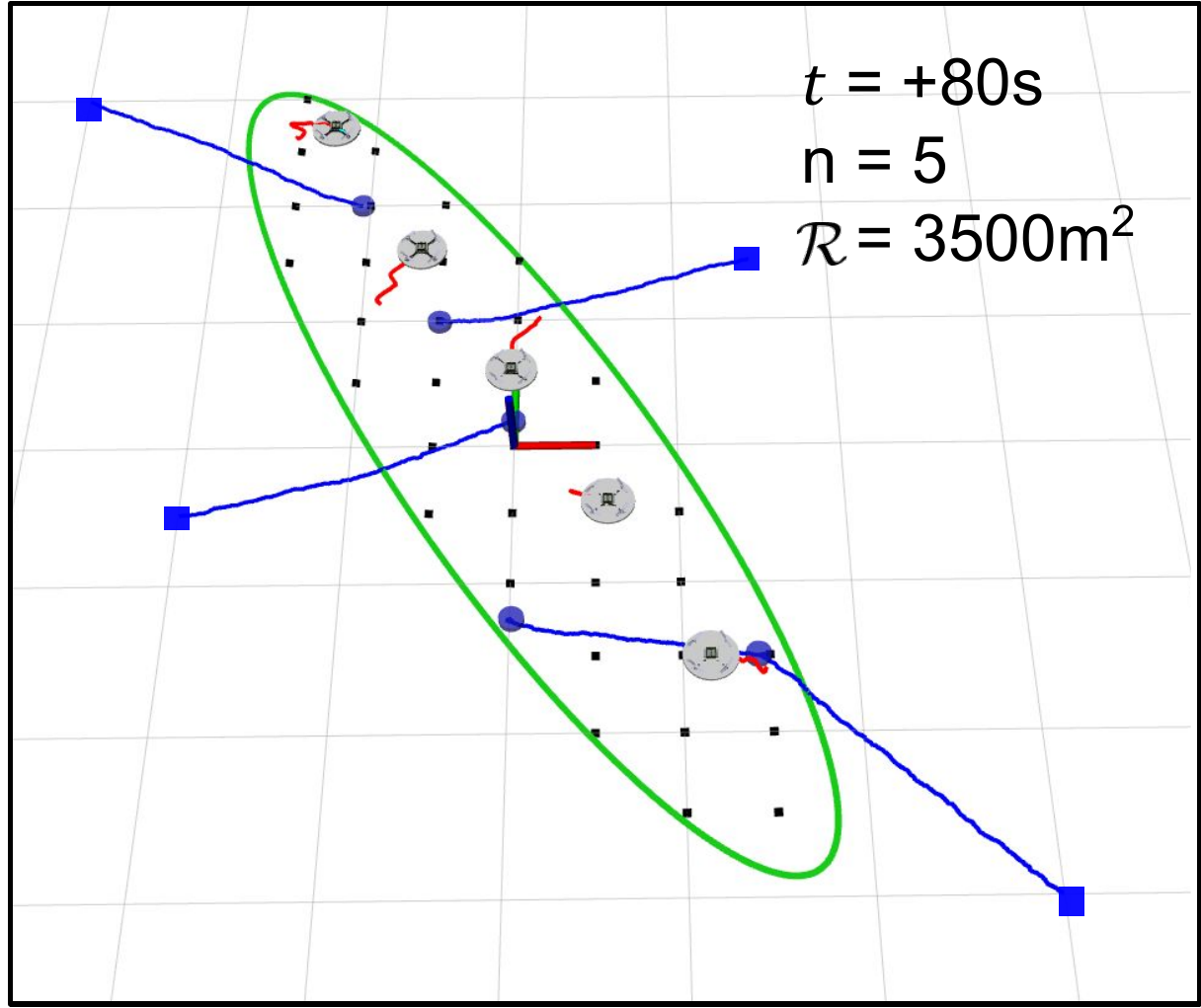}}
\subfigure[]
 	{\includegraphics[width=0.155\textwidth,trim={0.04cm 0.04cm 0.04cm 0.04cm}, clip]{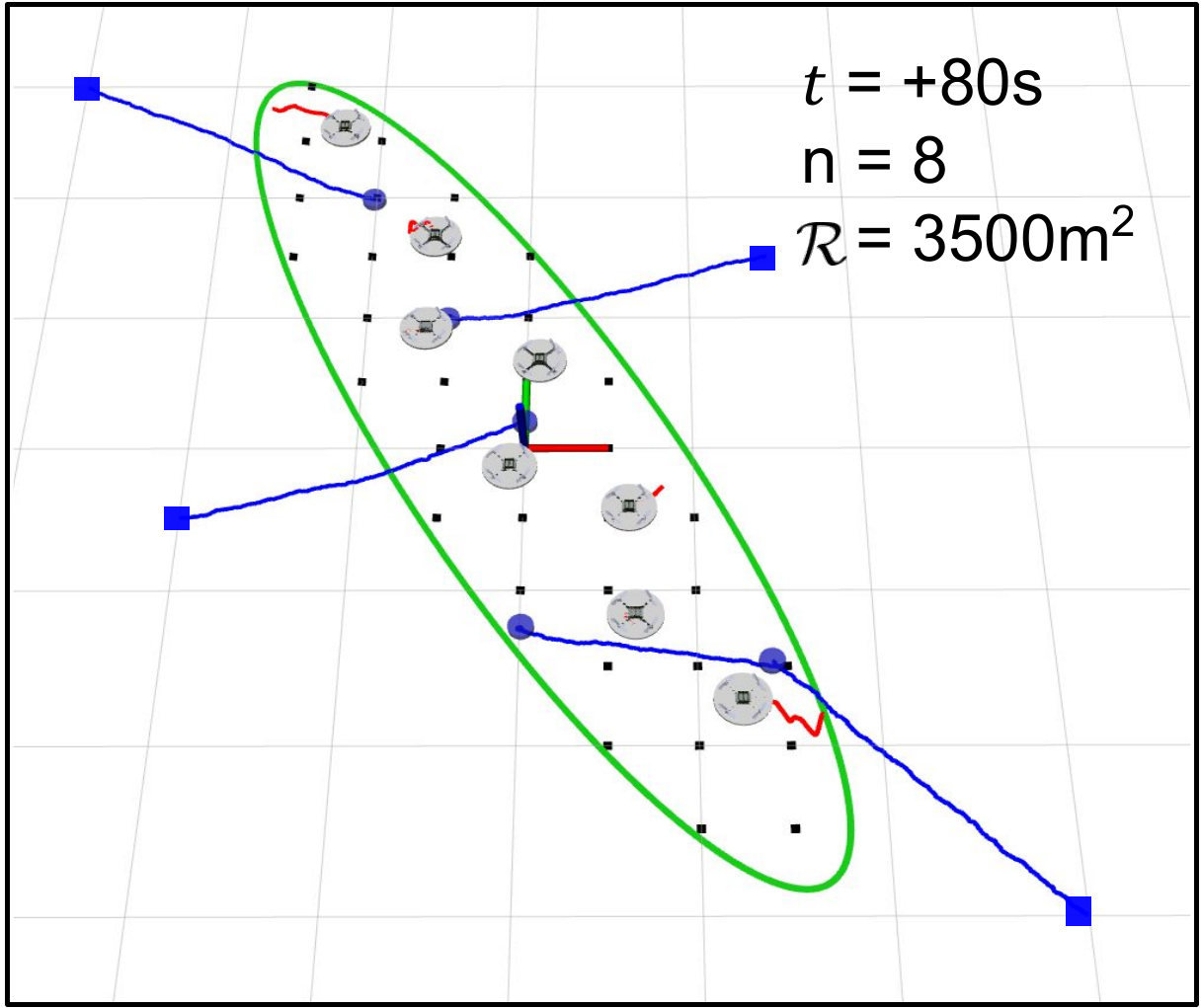}}
\caption{Trajectories of robot nodes with time. The top row shows the equilibrium communication topology for a stationary ROI and the initial UE positions (blue). Fig. (d)-(i) represents the UAV (red) and UE (blue) trajectories as the UE nodes move to their designated goal positions. The green ellipsoidal fence represents the concentration ellipsoid for the ROI $\mathcal{R}$. 
Note that the robot models are not to the scale for visualization purposes.}
 \label{fig:coco_images}
\end{figure}
Similarly, Fig. \ref{fig:coco_images} (d-i) show the UAV and UE trajectories for the moving ROI experiments.
The UEs travels between the start and arbitrarily chosen goal locations by morphing the ROI. 
The blue squares and dots represent the start and current locations of the UEs.
With the stage game optimization, the UAV swarm repeatedly move to maximize the coverage as the ROI changes. 
As the UE team members reach their destinations, the UAV swarm converge to the equilibrium formations as showed in Fig. \ref{fig:coco_images}(g-i).

 \begin{figure}[t!]
\centering
\subfigure[]
 	{\includegraphics[width = 0.241\textwidth, trim={0.1cm 0.2cm 0.8cm 0.4cm}, clip]{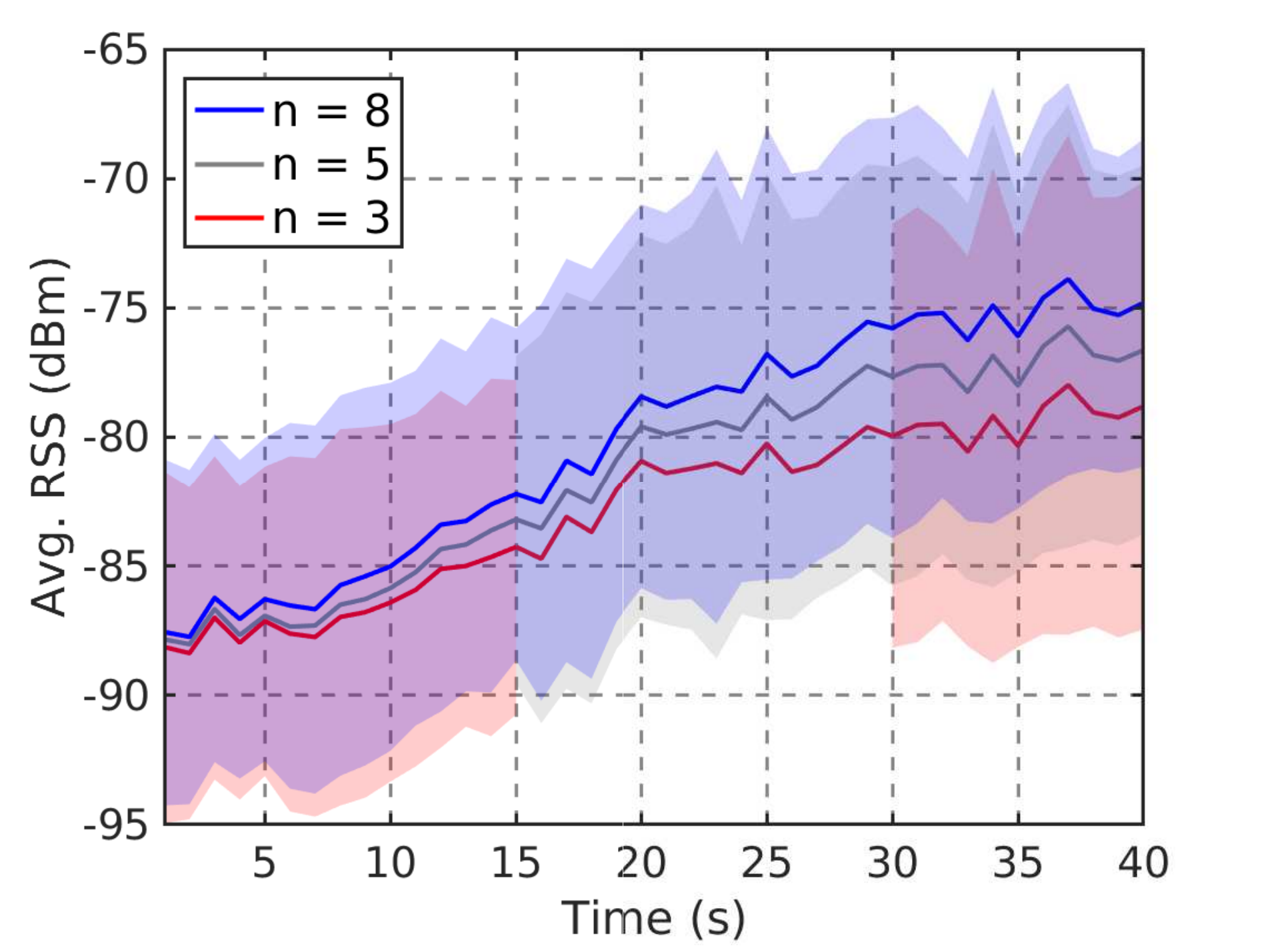}}
\subfigure[]
 	{\includegraphics[width = 0.241\textwidth, trim={0.1cm 0.1cm 0.8cm 0.4cm}, clip]{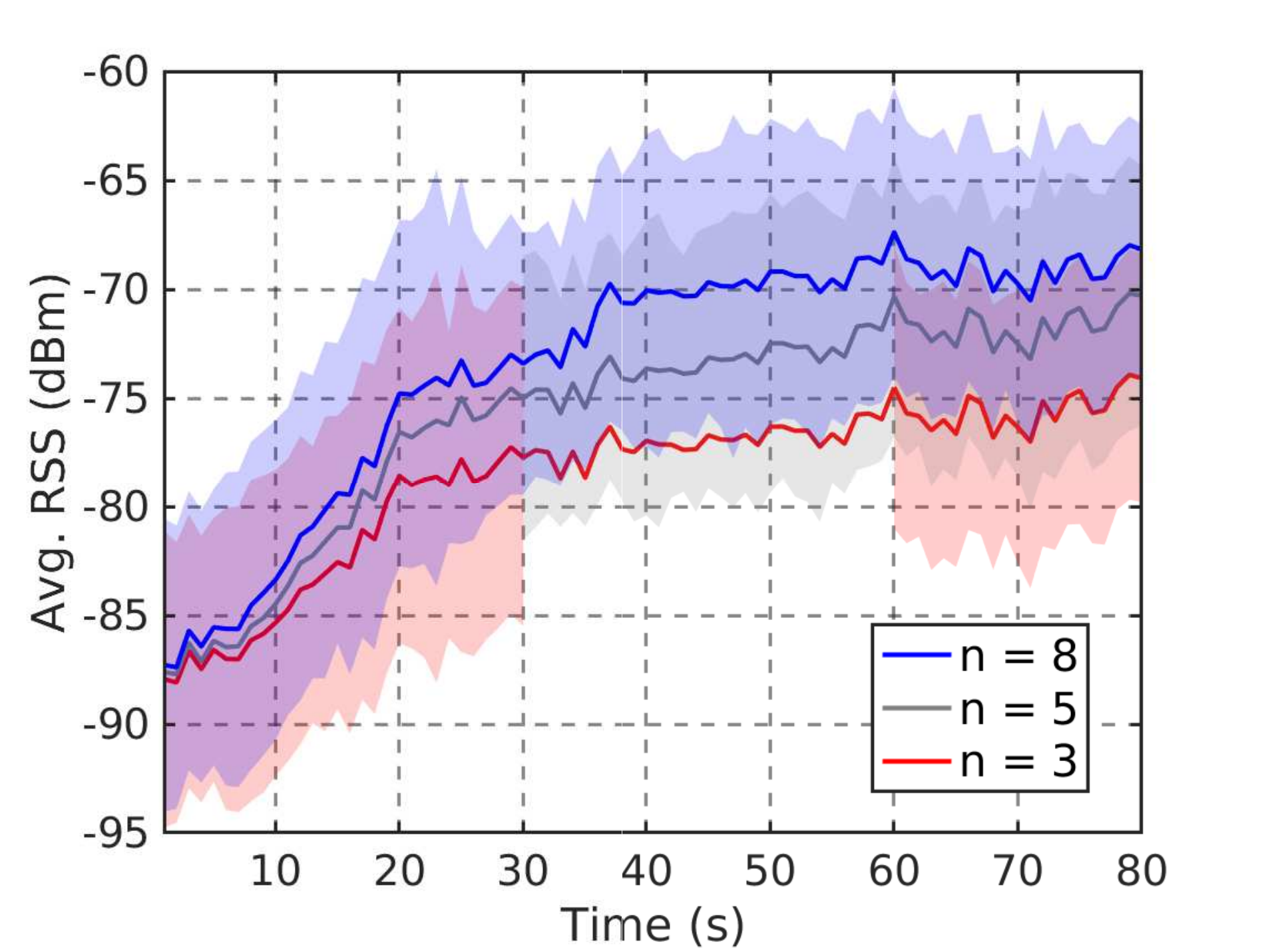}}
\caption{Change of the average RSS of UE nodes against time in stationary (a) and moving (b) scenarios for different UAV swarm sizes, using the CoCo model. For clarity, only the error for $n=8$ scenario is plotted throughout the experiments.}
 \label{fig:coco_plots}
  \vspace{-10pt}
\end{figure}
For clarity, only the most recent trajectory trails of the UAVs are showed.
We repeated each experiment for 20 trials using NS-3 to perform network events such as attenuating wireless signals, routing, and calculating RSS for each communication link.
In \ref{fig:coco_plots}(a-b), we present the average RSS as observed by a UE for for the experiments.
Throughout the convergence process, we observed that the UAVs to constantly increase their cooperative payoff, leading to higher RSS in the UEs.
We observed significant improvements in the average RSS as the UAVs reached the equilibrium formations under both scenarios.
Also, it's worth noting that decibel is a logarithmic scale, and every 3 unit increments double the signal strength.
The global connectivity of the network was maintained throughout the experiments, as the OLSR algorithm was able to find routes between the network nodes successfully.
The final communication topologies for the moving ROI experiments were observed to be more densely connected than the stationary scenarios due to the relatively small ROI size.

\rebut{We compare the network and optimization aspects of our framework to that of the widely used disk-based model and the decentralized, adaptive coverage method presented in \cite{schwager2009decentralized}.
Following our initial findings presented in Fig. \ref{fig:hops}(b), we selected the average single-hop distance $\approx 60$m as the radii for the disk model under similar environmental constraints.
We report the average observed RSS of a UE in Fig. \ref{fig:disk}(a-b) for the same configurations.
The disk-based models resulted in RSS gains initially. 
However, as the UAVs moved beyond each others' disk radii, the algorithm failed to reach a consensus formation.
Even though increasing the communication radii seems like a trivial solution, we emphasize that it can risk losing the global connectivity altogether due to the highly stochastic nature of the wireless signals at large distances.}

\rebut{
The decentralized, adaptive coverage control \cite{schwager2009decentralized} follows a similar approach to partition the environment based on the observed sensor gain for each robot.
For the comparison, we implemented the coverage as a bivariate Gaussian function defined at the center of $\mathcal{R}$ with covariance $\Sigma$.
The work, however, limits itself to static coverage functions and fails to handle the scenarios with dynamic ROIs and stochastic communication links.
In contrast to our unified framework, it also overlooks the robots' dynamics, thus the calculated paths needed to be further smoothed with trajectory optimization before deploying on the UAVs.
Fig. \ref{fig:disk}(c) shows the average RSS resulting in a UE using the decentralized, adaptive coverage control. 
Even though the approach converged to the consensus faster, it failed to yield higher RSS values in practice due to rigid connectivity and the overlooked dynamics.}
\begin{figure}[t]
\centering
\subfigure[]
 	{\includegraphics[width = 0.241\textwidth, trim={0.2cm 0.05cm 0.8cm 0.25cm}, clip]{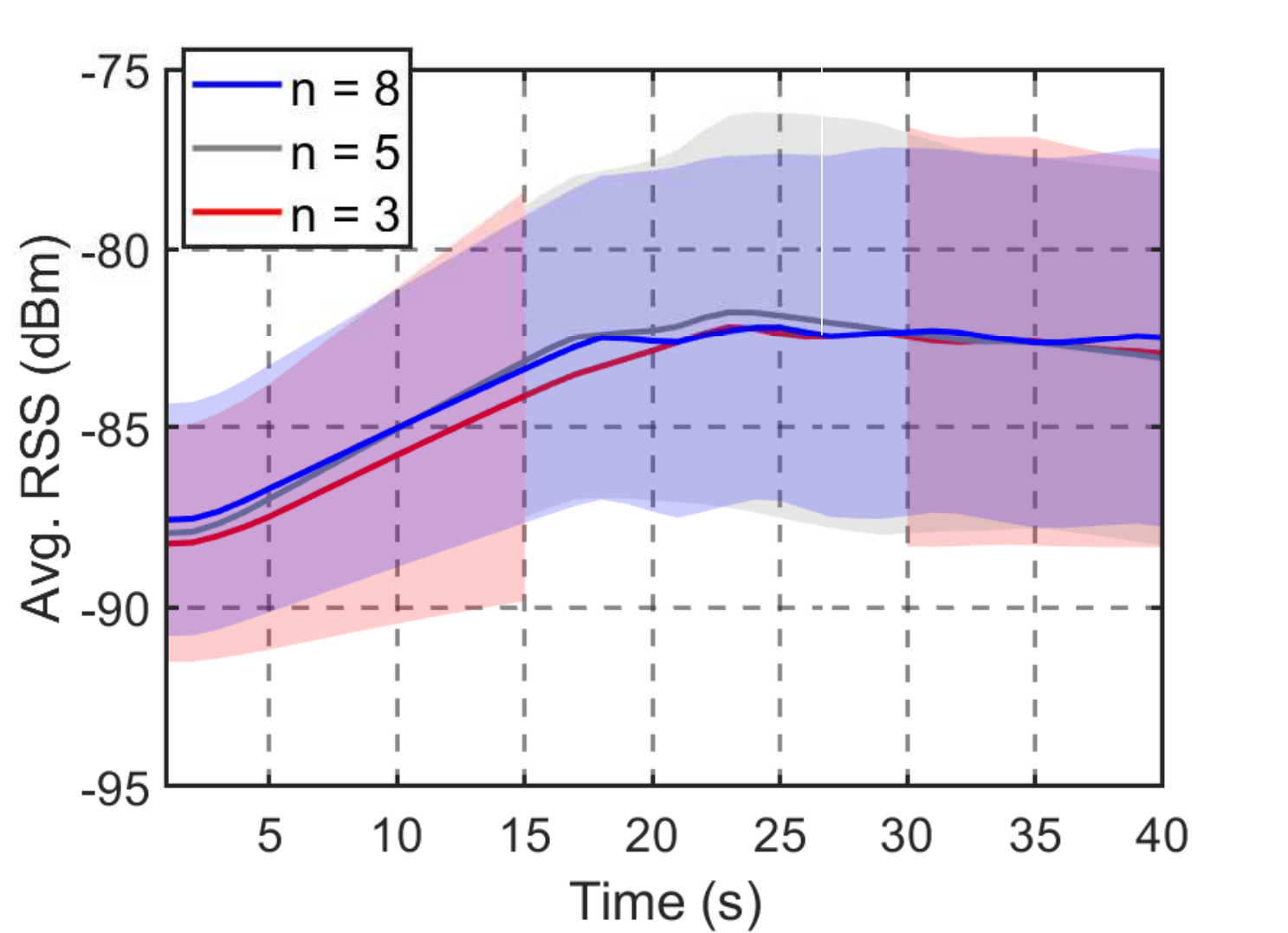}}
\subfigure[]
 	{\includegraphics[width = 0.241\textwidth, trim={0.2cm 0.05cm 0.8cm 0.25cm}, clip]{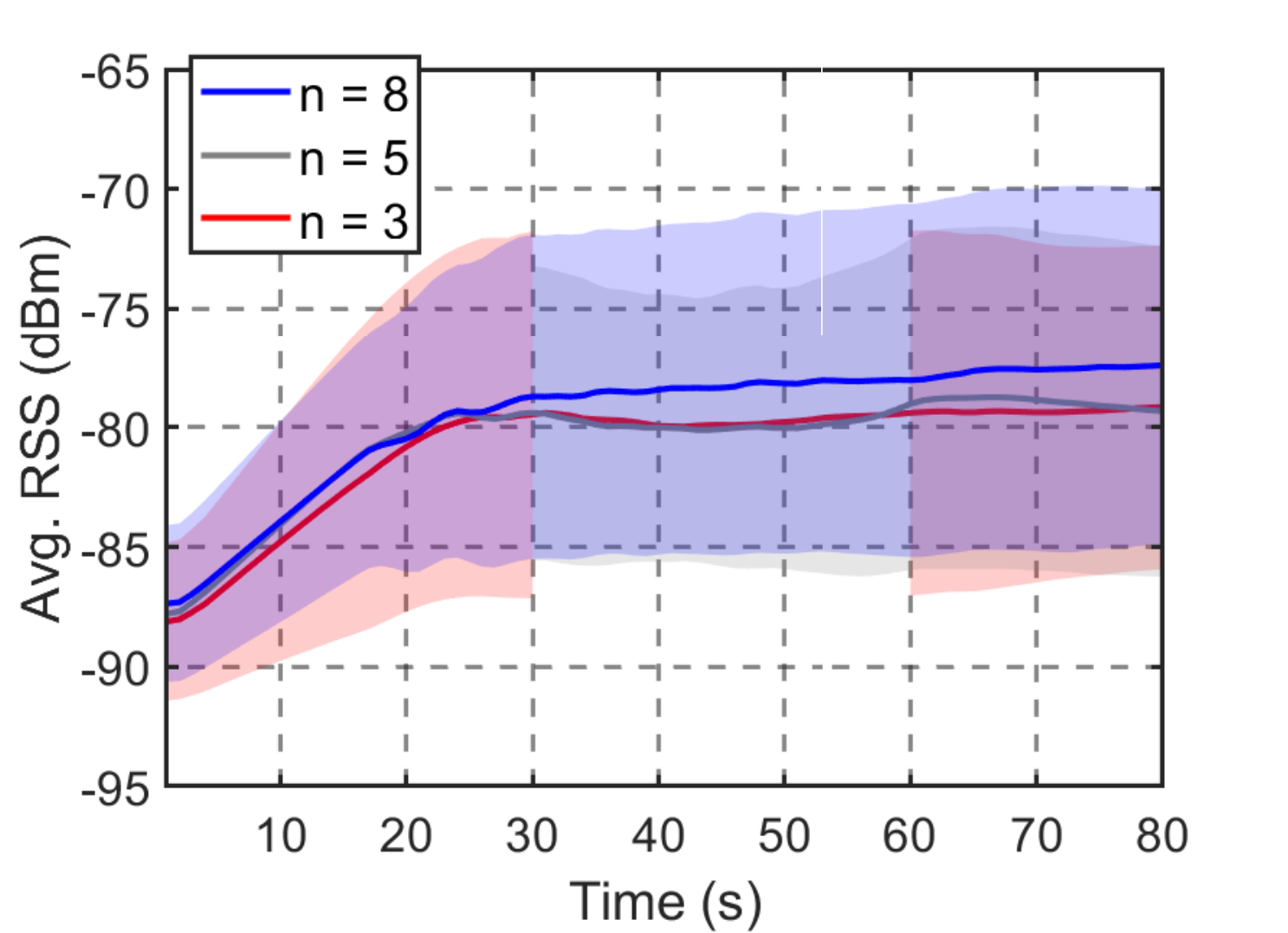}}
 \subfigure[]
 	{\includegraphics[width = 0.4\textwidth, trim={0cm 0.05cm 0.4cm 0.1cm}, clip]{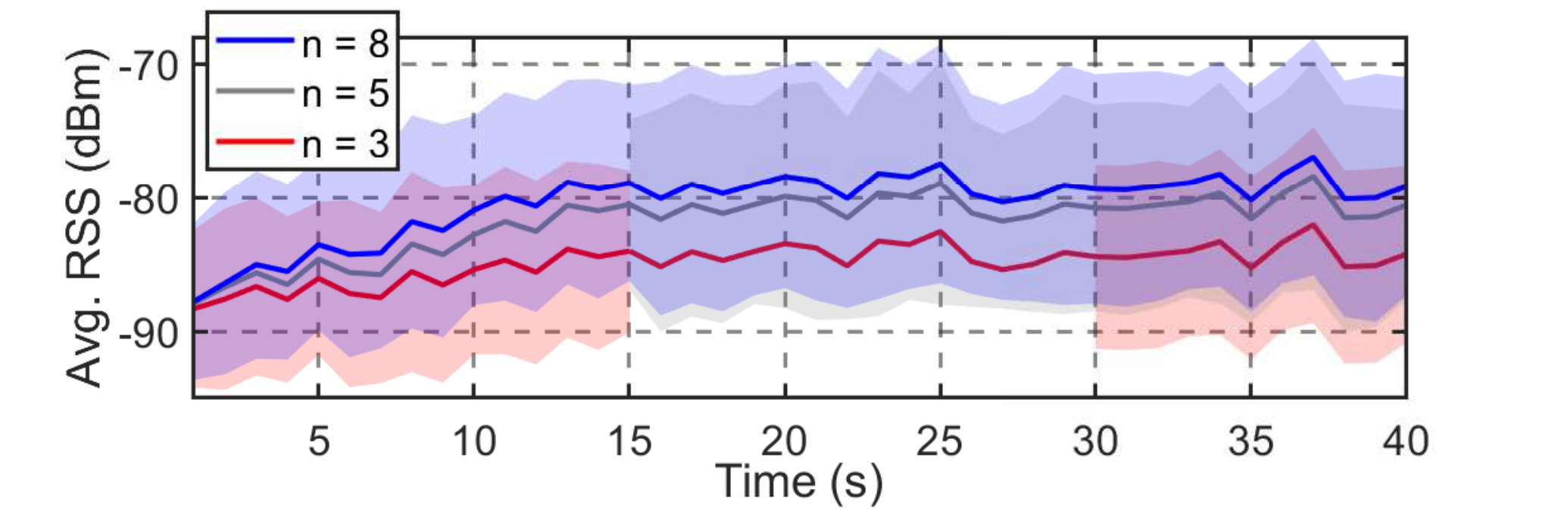}}
\caption{Comparisons: Average RSS of UE nodes in stationary (a) and moving ROI (b) scenarios for different UAV swarm sizes using a fixed-radii model. \rebut{(c) Average RSS plots for the adaptive coverage method in the stationary scenario.}}
 \label{fig:disk}
 \vspace{-10pt}
\end{figure}

\section{Conclusion}
\rebut{We have proposed a novel game-theoretic swarm coordination framework to achieve communication-aware coverage for robot swarms by only utilizing the local information.}
Our work complements the underlying network and robot dynamic models in neighborhood selection and control, resulting a robust coverage scheme for large-scale ROIs.
We have evaluated our approach in an ad-hoc mobile wireless network scenario to show that it can result in significant coverage gains while maintaining the local connectivity.
We have further established that our approach achieves real-time control and consensus for networked robot teams.
\bibliographystyle{ieeetr}
\bibliography{root}
\end{document}